\documentclass{article}

\usepackage[final,nonatbib]{neurips_2021}

\usepackage[utf8]{inputenc} % allow utf-8 input
\usepackage[T1]{fontenc}    % use 8-bit T1 fonts
\usepackage{hyperref}       % hyperlinks
\usepackage{url}            % simple URL typesetting
\usepackage{booktabs}       % professional-quality tables
\usepackage{amsfonts}       % blackboard math symbols
\usepackage{nicefrac}       % compact symbols for 1/2, etc.
\usepackage{microtype}      % microtypography
\usepackage{amsmath, amsfonts, amssymb, xspace, color, amsbsy, caption, adjustbox, amsthm}
\usepackage{tabularx}
\usepackage{tikz, graphicx}
\usepackage{enumitem}
\usepackage{multirow}
\usepackage{cleveref}
\usetikzlibrary{positioning}
\usepackage[linesnumbered,ruled,lined]{algorithm2e}
\usepackage[font=small,labelfont=bf]{caption}
\usepackage{subcaption}
\usepackage{esvect}

\usepackage{changepage}
\usepackage{pifont}
\theoremstyle{plain}
\newtheorem{definition}{Definition}[section]
\newtheorem{theorem}{Theorem}[section]

\newtheorem{lemma}[theorem]{Lemma}
\newtheorem{thm:eg}{Example}
\newtheorem{remark}{Remark}

\newcommand{\Ours}{\textsc{Egi}\xspace}
\newcommand{\ours}{\textsc{Egi}\xspace}
\newcommand{\xhdr}[1]{\vspace{1.7mm}\noindent{{\bf #1}}}

\newcommand{\ie}{\emph{i.e.}\xspace} % that is
\newcommand{\eg}{\emph{e.g.}\xspace} % for example
\newcommand{\etc}{\emph{etc.}\xspace} % that is
\newcommand{\wrt}{\emph{w.r.t.}\xspace} % with respect to
\newcommand{\iid}{\emph{i.i.d.}\xspace}

\newcommand{\ssym}{§}

\newcommand{\Sum}{\sum\limits} % Cheap displaystyle operators

\def \B {\mathcal{B}}

\def \D {\mathcal{D}}

\def \N {\mathcal{N}}

\def \calX {\mathcal{X}}

\def \calL {\mathcal{L}}

\def \RR {\mathbb{R}}

\def \sig {\sigma}

\def \gi  {g_i}
\def \gid {g_{i'}}

\def \gams {\gamma_{\sigma}}

\SetKwRepeat{Do}{do}{while}

\title{Transfer Learning of Graph Neural Networks with Ego-graph Information Maximization}

\begin{document}

\author{Qi Zhu$^1$\thanks{These two authors contribute equally.}, Carl Yang$^2$\footnotemark[1], Yidan Xu$^3$, Haonan Wang$^1$, Chao Zhang$^4$, Jiawei Han$^1$  \\
$^1$University of Illinois Urbana-Champaign, $^2$Emory University,\\
$^3$University of Washington, $^4$Georgia Institute of Technology \\
\texttt{$^1$\{qiz3,haonan3,hanj\}@illinois.edu, $^2$j.carlyang@emory.edu}, \\
\texttt{$^3$yx2516@uw.edu, $^4$chaozhang@gatech.edu}
}

\maketitle

\begin{abstract}
Graph neural networks (GNNs) have achieved superior performance in various applications, but training dedicated GNNs can be costly for large-scale graphs. Some recent work started to study the pre-training of GNNs. However, none of them provide theoretical insights into the design of their frameworks, or clear requirements and guarantees towards their transferability. In this work, we establish a theoretically grounded and practically useful framework for the transfer learning of GNNs. Firstly, we propose a novel view towards the \textit{essential graph information} and advocate the capturing of it as the goal of transferable GNN training, which motivates the design of \Ours (\textit{Ego-Graph Information maximization}) to analytically achieve this goal. Secondly,
when node features are structure-relevant,
we conduct an \textit{analysis of \Ours transferability} regarding the difference between the local graph Laplacians of the source and target graphs. We conduct controlled synthetic experiments to directly justify our theoretical conclusions. Comprehensive experiments on two real-world network datasets show consistent results in the analyzed setting of direct-transfering, while those on large-scale knowledge graphs show promising results in the more practical setting of transfering with fine-tuning.\footnote[1]{Code and processed data are available at \url{https://github.com/GentleZhu/EGI}.}
\end{abstract}
\section{Introduction}
\label{sec:intro}
Graph neural networks (GNNs) have been intensively studied recently
\cite{kipf2016semi, keriven2019universal, oono2020graph, yang2019conditional}, due to their established performance towards various real-world tasks \cite{hamilton2017inductive, ying2018hierarchical, velivckovic2017graph}, as well as close connections to spectral graph theory \cite{defferrard2016convolutional, bruna2013spectral, hammond2011wavelets}.
While most GNN architectures are not very complicated, the training of GNNs can still be costly regarding both memory and computation resources on real-world large-scale graphs \cite{chen2018fastgcn, yang2020multisage}. Moreover, it is intriguing to transfer learned structural information across different graphs and even domains in settings like few-shot learning \cite{vinyals2016matching, ravi2016optimization, kan2021zero}.
Therefore, several very recent studies have been conducted on the transferability
of GNNs %, which focus on the setting of pre-training plus fine-tuning 
\cite{hu2019strategies, hu2019pre, hu2020gpt,zhu2020shift,lan2020node,baek2020learning,ruiz2020graphon}.
%However, they do not provide theoretical analysis and guarantee towards the transferability of their models, and thus it is unclear in what situations the models will excel or fail.
However, it is unclear in what situations the models will excel or fail especially when the pre-training and fine-tuning tasks are different.
To provide rigorous analysis and guarantee on the transferability of GNNs, we focus on the setting of direct-transfering between the source and target graphs, under an analogous setting of ``domain adaptation''~\cite{ben2007analysis,wu2020unsupervised,zhu2020shift}.
%which is analogous to the ``domain adaption''~\cite{ben2007analysis} problem.
%However, except for conceptual examples and empirical studies, none of them attempts to argue about the requirements and guarantees towards the transferability of GNNs from a theoretical perspective, thus falling short in rigorously guiding the design of GNNs across various scenarios in practice.

In this work, we establish a theoretically grounded framework for the transfer learning of GNNs, and leverage it to design a practically transferable GNN model.
Figure \ref{fig:toy} gives an overview of our framework.
It is based on a novel view of a graph as samples from the joint
distribution of its k-hop ego-graph structures and node features, which allows us to define graph information and similarity, so as to analyze GNN transferability (\ssym\ref{sec:method}).
This view motivates us to design \Ours, a novel GNN training objective based on ego-graph information maximization, which is effective in capturing the graph information as we define (\ssym\ref{subsec:subgi}).
Then we further specify the requirement on transferable node features and analyze the transferability of \Ours that is dependent on the local graph Laplacians of source and target graphs (\ssym\ref{subsec:bound}).

All of our theoretical conclusions have been directly validated through controlled synthetic experiments (Table \ref{tab:synthetic}), where we use structural-equivalent role identification in an direct-transfering setting to analyze the impacts of different model designs, node features and source-target structure similarities on GNN transferability. 
In \ssym\ref{sec:experiment}, we conduct real-world experiments on multiple publicly available network datasets. On the Airport and Gene graphs (\ssym\ref{subsec:role}), we closely follow the settings of our synthetic experiments and observe consistent but more detailed results supporting the design of \Ours and the utility of our theoretical analysis. On the YAGO graphs (\ssym\ref{subsec:relation}), we further evaluate \Ours on the more generalized and practical setting of transfer learning with task-specific fine-tuning. We find our theoretical insights still indicative in such scenarios, where \Ours consistently outperforms state-of-the-art GNN representation and transfer learning frameworks with significant margins. 
%\carl{put a few numbers here}

\begin{figure}
  \centering
  %\fbox{\rule[-.5cm]{0cm}{4cm} \rule[-.5cm]{4cm}{0cm}}
  \includegraphics[width=0.8\textwidth]{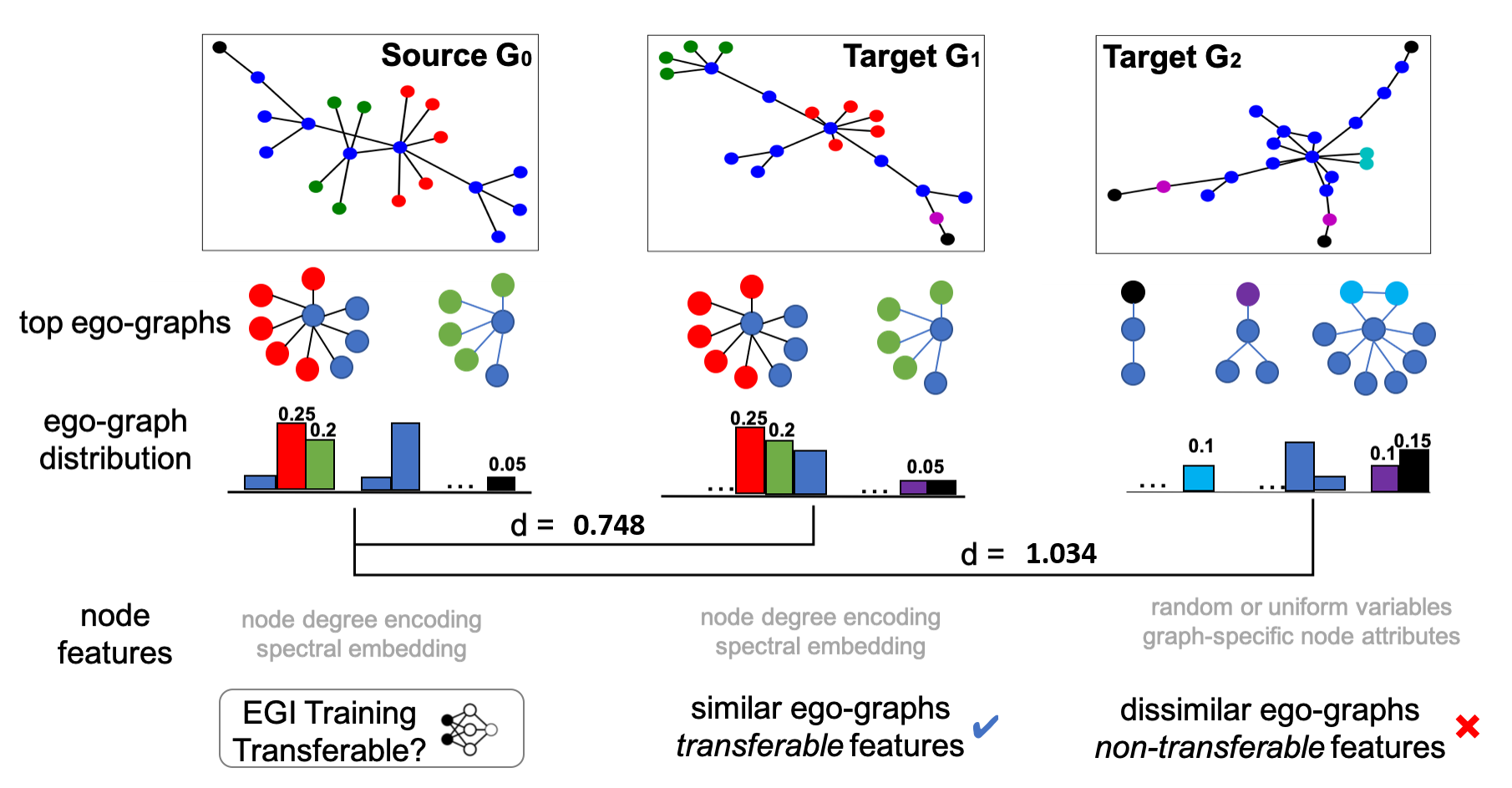}
  \caption{Overview of our GNN transfer learning framework: (1) we represent the toy graph as a combination of its 1-hop ego-graph and node feature distributions; (2) we design a transferable GNN regarding the capturing of such essential graph information; (3) we establish a rigorous guarantee of GNN transferability based on the node feature requirement and graph structure difference.}
  %the structural information of graphs with its ego-graph distributions, and then combine itThe ego-graph distribution is ordered by number of nodes in the ego-graph. The same color across three different graphs indicates isomorphism 2-hop ego-graphs. The unmatched ego-graphs are marked as background nodes (blue). 
  \label{fig:toy}
\end{figure}
\section{Related Work}
\label{related}
% graph, transductive embedding
Representation learning on graphs has been studied for decades, with earlier spectral-based methods \cite{belkin2002laplacian, roweis2000nonlinear, tenenbaum2000global} theoretically grounded but hardly scaling up to graphs with over a thousand of nodes. 
With the emergence of neural networks, unsupervised network embedding methods based on the Skip-gram objective \cite{mikolov2013distributed} have replenished the field \cite{tang2015line, grover2016node2vec, perozzi2014deepwalk, ribeiro2017struc2vec, yang2020co, yang2018meta, yang2018did}.
Equipped with efficient structural sampling (random walk, neighborhood, \etc) and negative sampling schemes, these methods are easily parallelizable and scalable to graphs with thousands to millions of nodes.
However, these models are essentially transductive as they compute fully parameterized embeddings only for nodes seen during training, which are impossible to be transfered to unseen graphs.

% gnn
More recently, researchers introduce the family of graph neural networks (GNNs) that are capable of inductive learning and generalizing to unseen nodes given meaningful node features \cite{kipf2016semi, defferrard2016convolutional, hamilton2017inductive, yang2020relation}.
Yet, most existing GNNs require task-specific labels for training in a semi-supervised fashion to achieve satisfactory performance \cite{kipf2016semi, hamilton2017inductive, velivckovic2017graph, yang2020heterogeneous}, and their usage is limited to single graphs where the downstream task is fixed. To this end, several unsupervised GNNs are presented, such as the auto-encoder-based ones like VGAE~\cite{kipf2016variational} and GNFs~\cite{liu2019graph}, as well as the deep-infomax-based ones like DGI~\cite{velivckovic2018deep} and InfoGraph~\cite{sun2019infograph}. Their potential in the transfer learning of GNN remains unclear when the node features and link structures vary across different graphs.

% transfer learning, transferrable gnn
Although the architectures of popular GNNs such as GCN \cite{kipf2016semi} may not be very complicated compared with heavy vision and language models, training a dedicated GNN for each graph can still be cumbersome \cite{chen2018fastgcn, yang2020multisage}. Moreover, as pre-training neural networks are proven to be successful in other domains \cite{devlin2018bert, he2016deep}, the idea is intriguing to transfer well-trained GNNs from relevant source graphs to improve the modeling of target graphs or enable few-shot learning \cite{wu2020unsupervised, lan2020node, baek2020learning} when labeled data are scarce. 
In light of this, pioneering works have studied both generative~\cite{hu2020gpt} and discriminative~\cite{hu2019strategies, hu2019pre} GNN pre-training schemes. 
%They empirically show some of them to be useful and others to be not. 
Though Graph Contrastive Coding~\cite{qiu2020gcc} shares the most similar view towards graph structures as us, it utilizes contrastive learning across all graphs instead of focusing on the transfer learning between any specific pairs. On the other hand, unsupervised domain adaptive GCNs~\cite{wu2020unsupervised} study the domain adaption problem only when the source and target tasks are homogeneous. 

Most previous pre-training and self-supervised GNNs lack a rigorous analysis towards their transferability and thus have unpredictable effectiveness.
The only existing theoretical work on GNN transferability studies the performance of GNNs across different permutations of a single original graph \cite{levie2019transferability, levie2019transferability2} and the tradeoff between discriminability and transferability of GNNs \cite{ruiz2020graphon}. We, instead, are the first to rigorously study the more practical setting of transferring GNNs across pairs of different source and target graphs.
%\QZ{Although there exists theoretical study on GNN transferability~, they study the transferability of GNN encoders and examine the performance on the permuted origin graph. We, instead, aim to study the transferability of GNN training objective including pre-training and unsupervised algorithms.}

% graph kernel, perturbation, theory, subgraph
%The transferability of models is essentially related to the model architectures and the similarity between source and target datasets or tasks.
%Unlike the language and image domains, graphs do not have a common vocabulary like words and pixels. 
%To model the similarity among graphs, exhaustive structural tests like WL-test \cite{weisfeiler1968reduction} have been used for decades to give 0-1 judgement to whether two graphs are isomorphic, \ie, having exactly same structures. To soften the comparison, researchers have studied the efficient construction of kernel spaces that capture graph structures regarding smaller sub-structures \cite{borgwardt2005shortest, shervashidze2011weisfeiler, shervashidze2009efficient, yanardag2015deep, bai2016subgraph}. To characterize indistinguishable structural differences, \cite{verma2019stability} studied the perturbations among graphs through their sub-graph distributions. Our way of using k-hop ego-graphs to model the structural similarity among graphs is closely related to all three lines above, but our leverage of it towards the guarantee of GNN transferability is unprecedented.

%\input{prelim}
\section{Transferable Graph Neural Networks}
\label{sec:method}
%We aim to establish a theoretically grounded and practically useful framework to argue about and quantify the transferability of GNN
%In particular, 
In this paper, we design a more transferable training objective for GNN (\Ours) based on our novel view of essential graph information (\ssym\ref{subsec:subgi}).
We then analyze its transferability as the gap between its abilities to model the source and target graphs, based on their local graph Laplacians (\ssym\ref{subsec:bound}).

Based on the connection between GNN and spectral graph theory \cite{kipf2016semi}, we describe the output of a GNN as a combination of its input node features $X$, fixed graph Laplacian $L$ and learnable graph filters $\Psi$. 
The goal of training a GNN is then to improve its utility by learning the graph filters that are compatible with the other two components towards specific tasks.

In the graph transfer learning setting where downstream tasks are often unknown during pre-training, we argue that the general utility of a GNN should be optimized and quantified \wrt~its ability of capturing the essential graph information in terms of the joint distribution of its topology structures and node features, which motivates us to design a novel ego-graph information maximization model (\Ours) (\ssym\ref{subsec:subgi}). % for efficient ego-graph reconstruction. Such utility is further reflectable on downstream tasks as long as the task is relevant to the graph structures .
The general transferability of a GNN is then quantified by the gap between its abilities to model the source and target graphs. %\ie, the compatibilities between its learned graph filters and its input node features and fixed graph Laplacians on the two graphs.
Under reasonable requirements such as using \textit{structure-respecting} node features as the GNN input, we analyze this gap for \Ours based on the structural difference between two graphs \wrt~their local graph Laplacians (\ssym\ref{subsec:bound}).

\subsection{Transferable GNN via Ego-graph Information Maximization}
\label{subsec:subgi}

In this work, we focus on the \textit{direct-transfering setting} where a GNN is pre-trained on a source graph $G_a$ in an unsupervised fashion and applied on a target graph $G_b$ without fine-tuning.\footnote{In the experiments, we show our model to be generalizable to the more practical settings with task-specific pre-training and fine-tuning, while the study of rigorous bound in such scenarios is left as future work.}
Consider a graph $G=\{V, E\}$, where the set of nodes $V$ are associated with certain features $X$ and the set of edges $E$ form graph structures.
Intuitively, the transfer learning will be successful only if both the features and structures of $G_a$ and $G_b$ are similar in some ways, so that the graph filters of a GNN learned on $G_a$ are compatible with the features and structures of $G_b$.

Graph kernels \cite{vishwanathan2010graph, borgwardt2020graph, kriege2020survey, nikolentzos2019graph} are well-known for their capability of measuring similarity between pair of graphs.
Motivated by k-hop subgraph kernels \cite{bai2016subgraph}, we introduce a novel view of a graph as \textit{samples from the joint distribution of its k-hop ego-graph structures and node features}.
Since GNN essentially encodes such k-hop ego graph samples, this view allows us to give concrete definitions towards \textit{structural information} of graphs in the transfer learning setting, which facilitates the measuring of similarity (difference) among graphs.
Yet, none of the existing GNN training objectives are capable of recovering such distributional signals of ego graphs.
To this end, we design \textit{Ego-Graph Information maximization} (\Ours), which alternatively reconstructs the k-hop ego-graph of each center node via mutual information maximization~\cite{hjelm2018learning}.
%Subsequently, we characterize an ideal transferable GNN as being able to capture the graph information regarding such joint distributions.

\begin{definition}[K-hop ego-graph]
We call a graph $g_i=\{V(g_i), E(g_i)\}$ a $k$-hop ego-graph centered at node $v_i$ if it has a $k$-layer centroid expansion \cite{bai2016subgraph} such that the greatest distance between $v_i$ and any other nodes in the ego-graph is k, \ie
$\forall v_j\in V(g_i), |d(v_i, v_j)| \leq k$, where $d(v_i,v_j)$ is the graph distance between $v_i$ and $v_j$. 
\label{def:k-hop}
\end{definition}
In this paper, we use directed k-hop ego-graph and its direction is decided by whether it is composed of incoming or outgoing edges to the center node, \ie, $g_i$ and $\tilde{g_i}$. The results apply trivially to undirected graphs with $g_i=\tilde{g_i}$.

%\QZ{plan to move these specifications into supps, it's not helpful for model understanding. For an ordered k-hop ego-graph, we denote $v_{p,q}$ as the $q$-th node in the $p$-th layer of the ego-graph (\ie, $|S_i(v_i, v_{p,q})|=p$), where $p=0,\ldots, k$, and $e_{vv'}$ as the edge between $v_{p, q}$ and $v_{p+1, q'}$.}
\begin{definition}[Structural information]
Let $\mathcal{G}$ be a topological space of sub-graphs,  
we view a graph $G$ as samples of k-hop ego-graphs $\{g_i\}_{i=1}^n$ drawn \iid from $\mathcal{G}$ with probability $\mu$, \ie, $g_i\overset{\iid}{\sim}\mu\;\forall i=1,\cdots, n$. 
The structural information of $G$ is then defined to be 
%the combination of the distribution $\mu$ and 
the set of k-hop ego-graphs of $\{g_i\}_{i=1}^n$ and their empirical distribution.
\label{def:structural}
\end{definition}

%The structural information of a graph $G$ can be characterized by $\{g_i\}_{v_i\in V}$ and its empirical distribution, with $V(g_i)$ and edges $E(g_i) = \{e_{uv}\in E(G): u,v\in V(g_i)\}$. 

As shown in Figure \ref{fig:toy}, three graphs $G_0$, $G_1$ and $G_2$ are characterized by a set of 1-hop ego-graphs and their empirical distributions, which allows us to quantify the structural similarity among graphs as shown in \ssym\ref{subsec:bound} (\ie, $G_0$ is more similar to $G_1$ than $G_2$ under such characterization).
In practice, the nodes in a graph $G$ are characterized not only by their k-hop ego-graph structures but also their associated node features. Therefore, $G$ should be regarded as samples $\{(g_i, x_i)\}$ drawn from the joint distribution $\mathbb{P}$ on the product space of $\mathcal{G}$ and a node feature space $\mathcal{X}$.
%In addition, we leverage a probabilistic view to jointly model ego-graph structures and node features towards the measuring of graph similarity.

% $\{g_i\}_{i=1}^n$ and node features $\{x_i\}_{i=1}^n$, \ie, $G=\{(g_i, x_i)\}_{i=1}^n, \;(g_i, x_i) \overset{\iid}{\sim} p\;\forall i=1,\cdots,n$, where $x_i$ is the set of node features on $g_i$.

%To capture such joint distributions of structural information and node features, we design \textit{ego-graph information maximization} (\Ours), which trains a GNN encoder $\Psi$ to maximize mutual information between structural information $g$ and output embeddings $z$.

%The main goal of \Ours is to capture the ego-graph distribution as an analogy to subtree graph kernels, which is transferable between source and target graphs.

%\QZ{reviewers suggest move the model architecture up. So I move the algorithm block into the supps, it's originally in supps.}
\begin{figure}[ht]
\centering
\includegraphics[width=0.9\textwidth]{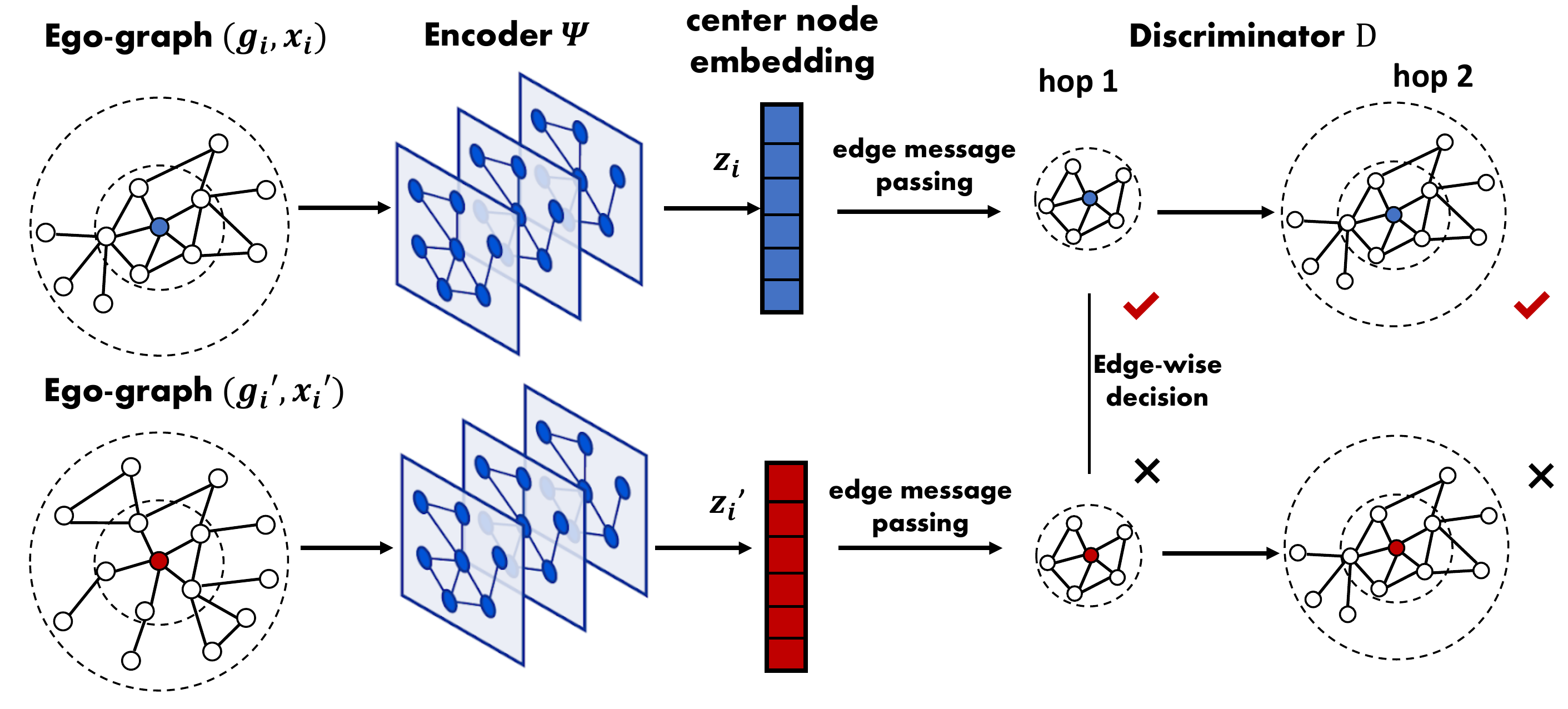}
\caption{The overall EGI training framework.}
\label{fig:framework}
\end{figure}

\xhdr{Ego-Graph Information Maximization.}
Given a set of ego-graphs $\{(g_i, x_i)\}_i$ 
%as defined in Def~\ref{def:structural} 
drawn from an empirical joint distribution $(g_i, x_i) \sim \mathbb{P}$. 
We aim to train an GNN encoder $\Psi$ to maximize the mutual informaion (MI $(g_i,\Psi(g_i, x_i))$) between the defined structural information $g_i$\footnote{Later in section~\ref{subsec:bound}, we will discuss the equivalence between MI($g_i,z_i$) and MI($(g_i,x_i),z_i$) when node feature is structure-respecting.} (\ie k-hop ego-graph) and node embedding $z_i = \Psi(g_i, x_i)$. To maximize the MI, another discriminator $\D (g_i, z_i) : E(g_i) \times z_i \rightarrow \mathbb{R}^{+}$ is introduced to compute the probability of an edge $e$ belongs to the given ego-graph $g_i$.
%Similar with the ``local'' version of DIM~\cite{hjelm2018learning}, we define $\mathbb{U}_{\Psi(g_i, x_i)}$ as the empirical distribution of the embedding produced by the GNN encoder $\Psi$ for the the center node $v_i$ of ego-graph $g_i$.
%Unlike DGI~\cite{velivckovic2018deep} that models the local-global mutual information (MI), \Ours optimizes $\Psi$ to maximize the MI $\mathcal{I}(g_i,\Psi(g_i, x_i))$  with a focus on the structural information $g_i$
We use the Jensen-Shannon MI estimator \cite{hjelm2018learning} in the \Ours objective,
\begin{equation}
\label{eq:EGI}
    \begin{array}{r}
    \mathcal{L}_\Ours = -
    \text{MI}^{\text{(JSD)}}\left(\mathcal{G},\Psi\right) = 
    \frac{1}{N} \Sum_{i=1}^N \left[ \text{sp} \left( \D (g_i, z_i^\prime) \right) + \text{sp}\left( - \D (g_i, z_i) \right) \right],
    %\mathbb{E}_{\mathbb{P}} \left[-\text{sp} \left( -T_{\D, \Psi}(g_i, \Psi(g_i, x_i) ) \right) \right]  - \mathbb{E}_{\mathbb{P} \times \tilde{\mathbb{U}}} \left[ \text{sp} \left( T_{\D, \Psi} (g_i, \Psi(g_i^\prime, x_i^\prime) ) \right) \right],
    \end{array}
\end{equation}
%where $(g_i^\prime, x_i^\prime)$ is a negative ego-graph sampled from $\mathbb{P}$;
%where $T_{\D, \Psi} = \D \left( g_i, \Psi(g_i, x_i) \right)$ and $\D$ is a discriminator $\mathcal{D}: g_i \times \Psi(g_i, x_i) \rightarrow \mathbb{R}^{+}$.
where $\text{sp}(x)= \log(1+e^x)$ is the softplus function and $(g_i, z_i^\prime)$ is randomly drawn from the product of marginal distributions, \ie $z_i^\prime = \Psi(g_{i^\prime}, x_{i^\prime}), (g_{i^\prime}, x_{i^\prime})\sim \mathbb{P}, i^\prime \ne i $. In general, we can also randomly draw negative $g_i^\prime$ in the topological space, while enumerating all possible graphs $g_{i^\prime}$ leads to high computation cost.

In Eq.~\ref{eq:EGI}, the computation of $\D$ on $E(g_i)$ depends on the node orders. Following the common practice in graph generation \cite{you2018graphrnn}, we characterize the decision process of $\D$ with a fixed graph ordering, \ie, the BFS-ordering $\pi$ over edges $E(g_i)$. $\D = f \circ \Phi$ is composed by another GNN encoder $\Phi$ and scoring function $f$ over an edge sequence $E^\pi: \{e_1, e_2, ..., e_n\}$, which makes predictions on the BFS-ordered edges.
%The input space of $\D$ can be as large as the number of graph permutations $|V(g_i)|!$. 

%Instead of enumerating all possible graphs $g_i^\prime$, we fix $g_i$ and sample GNN's output $\Psi(g_i^\prime, x_i^\prime)$ from the marginal distribution $\tilde{\mathbb{U}}$ by uniformly sampling $(g_i^\prime, x_i^\prime) \sim \tilde{\mathbb{P}},  \tilde{\mathbb{P}}\overset{d}{=} \mathbb{P}$.

Recall our previous definition on the direction of k-hop ego-graph, the center node encoder $\Psi$ receives pairs of $(g_i,x_i)$ while the neighbor node encoder $\Phi$ in discriminator $\D$ receives $(\tilde{g_i},x_i)$. Both encoders are parameterized as GNNs,
$$
\Psi(g_i,x_i) = \text{GNN}_{\Psi} (A_i, X_i), \Phi(\tilde{g_i},x_i) = \text{GNN}_{\Phi} (A_i^\prime, X_i),
$$
where $A_i, A_i^\prime$ is the adjacency matrix with self-loops of $g_i$ and $\tilde{g_i}$, respectively. The self-loops are added following the common design of GNNs, which allows the convolutional node embeddings to always incorporate the influence of the center node. $A_i = {A_i^\prime}^\intercal$. The output of $\Psi$, \ie, $z_i \in \mathbb{R}^n$, is the center node embedding, while $\Phi$ outputs representation $H \in \mathbb{R}^{|g_i| \times n}$ for neighbor nodes in the ego-graph.

%\yd{Do we need the following sentence? We have not mention structural respecting yet! There does not appear to be a following discussion on structrual feature neither. I think just remove it. }The correspondence between sampling $(g_i^\prime, x_i^\prime) \sim \tilde{\mathbb{P}}$ and $g_i^\prime \sim \mathcal{G}$ is discussed in Remark~\ref{remark:feature} when node features are strcuture-respecting (Def.~\ref{def:respecting}).

%we maximize the MI by discriminating the samples from joint probability distribution $\mathbb{P}(G_s, f_{G_s}(x_i) )$ and product of marginals $\mathbb{P}(G_s) \times \mathbb{P}(f_{G_s}(x_i))$ by a discriminator $\mathcal{D}: G_s \times  f_{G_s}(x_i) \rightarrow \mathbb{R}$:
%Next, we present the design of our GNN discriminator $\D$ to cope with the samples of $(g_i, x_i)$'s.  
%\carl{here you may want to introduce your negative sampling over node embeddings, which help explain eq 3, and can also hint on def 3.3 and remark 2}
%Empirically, 

% $s_i \in \{0,1\}$ represents whether there is an edge between $u_i$ and $v_i$.

Once node representation $H$ is computed, we now describe the scoring function $f$. For each of the node pair $(p,q) \in E^\pi$, $h_p$ is the source node representation from $\Phi$, $x_q$ is the destination node features. The scoring function is,
%$\mathcal{M}(s_i |x_i), x_i = (g_i, x_i), \Psi(g_i, x_i)$ .
%as joint probability of observing each edges regarding a fixed permutation $\pi$, \ie BFS-ordering. For a k-hop subgraph $G_s$,
%where $h$ is the hidden representation output by $\D$, $e_{\tilde{v}v} \in E(g_i)$ is an edge between node $\tilde{q}$ in layer $p$ and $q$ in layer $p+1$, following the notation defined below Def~\ref{def:k-hop}.
%More specifically, we have
\begin{equation}
    %\D(e_{vv'}|h_{p,q}, x_{p,q}, z_i) = \sigma \left( U^T \cdot \tau \left( W^T [ h_{p,q}|| x_{p,q}|| z_i] \right) \right),
    f(h_p, x_q, z_i) = \sigma \left( U^T \cdot \tau \left( W^T [ h_p|| x_q|| z_i] \right) \right),
    %\D(e_{\tilde{v}v}|h_{p,q}^{i}, x^i_{p,q}, z_i) = \sigma \left( U^T \cdot \tau \left( W^T [ h_{p,q}^{i}|| x^i_{p,q}|| z_i] \right) \right),
\end{equation}        
where $\sigma$ and $\tau$ are Sigmoid and ReLU activation functions. Thus, the discriminator $\D$ is asked to distinguish a positive $((p,q), z_i)$ and negative pair $((p,q), z_i^\prime))$ for each edge in $g_i$.
\begin{equation}
\label{eq:2}
    \mathcal{D} (g_i, z_i) = \Sum_{(p,q) \in {E^\pi}} \log f(h_p, x_q, z_i), \ \ \mathcal{D} (g_i, z_i^\prime) = \Sum_{(p,q)}^{E^\pi} \log f(h_p, x_q, z_i^\prime).
    %\D(e_{vv'}|h_{p,q}, x_{p,q}, z_i),
    %\D(e_{\tilde{v}v}|h_{p,q}^{i}, x^i_{p,q}, z_i),
\end{equation}
%from GNN encoder output  $\Psi(g_i, x_i) \sim \mathbb{U}$ or $\Psi(g_i^\prime, x_i^\prime) \sim \tilde{\mathbb{U}}$
%\yd{what's $h_u$ and $z_v$ here?}\carl{and what does it mean by distinguishing the two? explain in plain language?}
%Thus, the mutual information $\mathcal{I}((g_i, x_i),\Psi(g_i, x_i))$ is calculated as a summation of its edges $E(g_i)$

%In the next section, the transferability bound
There are two types of edges $(p,q)$ in our consideration of node orders, \textit{type-a} - the edges across different hops (from the center node), and \textit{type-b} - the edges within the same hop (from the center node). The aforementioned BFS-based node ordering guarantees that Eq.~\ref{eq:2} is sensitive to the ordering of type-a edges, and invariant to the ordering of type-b edges, which is consistent with the requirement of our theoretical analysis on $\Delta_\mathcal{D}$.
Due to the fact that the output of a k-layer GNN only depends on a k-hop ego-graph for both encoders $\Psi$ and $\Phi$, \Ours can be trained in parallel by sampling batches of $g_i$'s. Besides, the training objective of \Ours is transferable as long as $(g_i, x_i)$ across source graph $G_a$ and $G_b$ satisfies the conditions given in \ssym\ref{subsec:bound}. %\yd{This does not logically sounds right. We analysed the transferability of the objective with two different graph, and the main take away is that we need two graphs similar enough, iin terms of minimizing the in-degree term.}. 
More model details in Appendix \ssym \ref{supp:model} and source code in the Supplementary Materials.

%\QZ{todo} Add illustration here how to train a GNN discriminator
%\QZ{add structure respecting feature requirements}

\xhdr{Connection with existing work.} To provide more insights into the \Ours objective, we also present it as a dual problem of ego-graph reconstruction. Recall our definition of ego-graph mutual information MI$(g_i,\Psi(g_i, x_i))$. It can be related to an ego-graph reconstruction loss $R(g_i|\Psi(g_i, x_i))$ as
\begin{equation}
    \begin{array}{r}
    \max \text{MI}(g_i,\Psi(g_i, x_i)) = H(g_i) - H(g_i|\Psi(g_i, x_i))  \leq H(g_i) - R(g_i|\Psi(g_i, x_i)).
    \end{array}
\end{equation}
When \Ours is maximizing the mutual information, it simultaneously minimizes the upper error bound of reconstructing an ego-graph $g_i$.
In this view, the key difference between \Ours and VGAE~\cite{kipf2016variational} is they assume each edge in a graph to be observed independently during the reconstruction. While in \Ours, edges in an ego-graph are observed jointly during the GNN decoding.
%Since we reconstruct the ego-graphs in a BFS ordering, our discriminator $\D$ is also analogous to sequential graph generation models like GraphRNN~\cite{you2018graphrnn} and GatedGNN~\cite{li2015gated}.
Moreover, existing mutual information based GNNs such as DGI~\cite{velivckovic2018deep} and GMI~\cite{peng2020graph} explicitly measure the mutual information between node features $x$ and GNN output $\Psi$. In this way, they tend to capture node features instead of graph structures, which we deem more essential in graph transfer learning as discussed in \ssym\ref{subsec:bound}.

\xhdr{Use cases of \Ours framework.}
In this paper, we focus on the classical domain adaption (direct-transferring) setting \cite{ben2007analysis}, where no target domain labels are available and transferability is measured by the performance discrepancy without fine-tuning. 
In this setting, the transferability of \ours is theoretically guaranteed by Theorem \ref{theo:main}.
In \ssym\ref{subsec:role}, we validated this with the airport datasets. Beyond direct-transferring, \ours is also useful in the more generalized and practical setting of transfer learning with fine-tuning, which we introduced in \ssym\ref{subsec:relation} and validated with the YAGO datasets. In this setting, the transferability of \ours is not rigorously studied yet, but is empirically shown promising.

\xhdr{Supportive observations.} In the first three columns of our synthetic experimental results (Table \ref{tab:synthetic}), in both cases of transfering GNNs between similar graphs (F-F) and dissimilar graphs (B-F), \Ours significantly outperforms all competitors when using node degree one-hot encoding as transferable node features. In particular, the performance gains over the untrained GIN show the effectiveness of training and transfering, and our gains are always larger than the two state-of-the-art unsupervised GNNs. Such results clearly indicate advantageous structure preserving capability and transferability of \Ours.

\subsection{Transferability analysis based on local graph Laplacians}
\label{subsec:bound}
We now study the transferability of a GNN (in particular, with the training objective of $\mathcal{L}_\Ours$) between the source graph $G_a$ and target graph $G_b$ based on their graph similarity.
% In \ssym\ref{subsec:subgi}, we introduced a novel view of a graph as a sample from the joint distribution of its link topology and node features.
We firstly establish the requirement towards node features, under which we then focus on analyzing the transferability of \Ours \wrt~the structural information of $G_a$ and $G_b$.

%Intuitively, if the node features on $G_a$ and $G_b$ are different in nature, \eg, visual features on $G_a$ and textual features on $G_b$, the transferability of GNNs is hard to study. Therefore, in this work, we first require the node features on two graphs to be the same in nature.
Recall our view of the GNN output as a combination of its input node features, fixed graph Laplacian and learnable graph filters. The utility of a GNN is determined by the compatibility among the three. In order to fulfill such compatibility, we require the node features to be \textit{structure-respecting}:
\begin{definition}[Structure-respecting node features]
Let $g_i$ be an ordered ego-graph centered on node $v_i$ with a set of node features $\{x^i_{p,q}\}_{p=0, q=1}^{k, |V_p(g_i)|}$, where $V_p(g_i)$ is the set of nodes in $p$-th hop of $g_i$. Then we say the node features on $g_i$ are structure-respecting if $x_{p,q}^i = [f(g_i)]_{p,q}\in\RR^d$ for any node $v_q\in V_p(g_i)$, where $f:\mathcal{G}\to \RR^{d\times |V(g_i)|}$ is a function. In the strict case, $f$ should be injective.
\label{def:respecting}
\end{definition}
In its essence, Def \ref{def:respecting} requires the node features to be a function of the graph structures, which is sensitive to changes in the graph structures, and in an ideal case, injective to the graph structures (\ie, mapping different graphs to different features). In this way, when the learned graph filters of a transfered GNN is compatible to the structure of $G$, they are also compatible to the node features of $G$. As we will explain in Remark \ref{remark:feature} of Theorem \ref{theo:main}, this requirement is also essential for the analysis of \Ours transferability which eventually only depends on the structural difference between two graphs.

In practice, commonly used node features like node degrees, PageRank scores \cite{page1999pagerank}, spectral embeddings \cite{chung1997spectral}, and many pre-computed unsupervised network embeddings \cite{perozzi2014deepwalk, tang2015line, grover2016node2vec} are all structure-respecting in nature. However, other commonly used node features like random vectors \cite{yang2019conditional} or uniform vectors \cite{xu2019powerful} are not and thus non-transferable.
When raw node attributes are available, they are transferable as long as the concept of \textit{homophily} \cite{mcpherson2001birds} applies, which also implies Def \ref{def:respecting}, but we do not have a rigorous analysis on it yet. 
%We will show in \ssym\ref{sec:experiment} that although our Theorem \ref{theo:main} does not exactly apply to the case of using organic node attributes as features, the bound is still indicative since homophily applies.

\xhdr{Supportive observations.} In the fifth and sixth columns in Table \ref{tab:synthetic}, where we use same fixed vectors as non-transferable node features to contrast with the first three columns, there is almost no transferability (see $\delta(acc.)$) for all compared methods when non-transferable features are used, as the performance of trained GNNs are similar to or worse than their untrained baselines. More detailed experiments on different transferable and non-transferable features can be found in Appendix \ssym \ref{supp:exp_syn}.

With our view of graphs and requirement on node features both established, now we derive the following theorem by characterizing the performance difference of \Ours on two graphs based on Eq.~\ref{eq:EGI}.
\begin{theorem}[GNN transferability]
Let $G_a=\{(g_i, x_i)\}_{i=1}^n$ and $G_b=\{(\gid, x_{i'})\}_{i'=1}^m$ be two graphs, and
% where$\gid, \gi\overset{\iid}{\sim} \mu$, $x_i, z_{i'}\overset{\iid}{\sim} \nu$, and $(g_i, x_i), (\gid, z_{i'})\overset{\iid}{\sim} p$.
assume node features are structure-relevant.
%, where $x_{p,q}^i = [h(g_i)]_{p,q}\in\SSS^{d-1}\;\forall p,q$.
Consider GCN $\Psi_{\theta}$ with k layers and a 1-hop polynomial filter $\phi$.
% that is $\gamma_{\theta}$-Lipschitz.
% and is denoted in the node-centered form,  $F_{\theta}(x)=\sigma(\sum_{j\in\mathcal{N}} e_{\cdot j}z_x^{1(j)})$, where $e_{\cdot j}=[\phi_{\theta}(L)]_{\cdot j}\in\RR$.
% Finally, assume that $\|\phi_{\theta}(L_{g_i})\|_2\leq M<\infty$, for $L_{g_i}$ the graph Laplacian of $g_i\in G$.
With reasonable assumptions on the local spectrum of $G_a$ and $G_b$, the empirical performance difference of $\Psi_{\theta}$ evaluated on $\calL_{\Ours}$ satisfies
\begin{equation}
\begin{array}{l}
% \Delta \calL_{\Ours}(G_a, G_b, \mathcal{D}) \leq \\
%|\calL_{\Ours}(G_a) - \calL_{\Ours}(G_b)|\leq \\ 
|\calL_{\Ours}(G_a) - \calL_{\Ours}(G_b)|\leq
\mathcal{O}
\left(
    \Delta_{\mathcal{D}}(G_a, G_b)+C
\right).
\end{array}
\label{eq:theo}
\end{equation}
On the RHS, $C$ is only dependent on the graph encoders and node features, while $\Delta_{\mathcal{D}}(G_a, G_b)$ measures the structural difference between the source and target graphs as follows,
\begin{align}
\Delta_{\mathcal{D}}(G_a, G_b) = 
\Tilde{C}\frac{1}{nm}
    \sum_{i=1}^{n}\sum_{i'=1}^{m}
    \lambda_{\max}(
    \Tilde{L}_{g_i} - 
    \Tilde{L}_{g_{i'}})
\label{eq:gap}
\end{align}
where $\lambda_{\max}(A):=\lambda_{\max}(A^TA)^{1/2}$, and $\Tilde{L}_{g_i}$ denotes the normalised graph Laplacian of $\Tilde{g_i}$ by its in-degree. $\Tilde{C}$ is a constant dependant on $\lambda_{\max}(\Tilde{L}_{g_i})$ and $\D$.
% where $M$ is a constant dependant on $k$, $\phi_{\theta}$, $\{L_{\gi}\}$, $\{L_{\gid}\}$, $\{x_i\}$, $\{x_{i'}\}$, and
% In addition, if $\exists U\in O(n\vee m)$\footnote{$O(n\vee m)$ is the orthogonal group of order $n\vee m$. So we have $L_{\gi}$ and $L_{\gid}$ admits simultaneous ordered spectral decomposition.} s.t.,
% $$
% UL_{\gi}U^T = \text{Diag}(\lambda(L_{\gi})), \quad UL_{\gid}U^T = \text{Diag}(\lambda(L_{\gid}))
% $$
% we have $\mathcal{O}
% \left(M + \frac{1}{nm}
%     \sum_{i=1}^{n}\sum_{i'=1}^{m}
%     \|\lambda(L_{\gi}) - \lambda(L_{\gid})\|_2
% \right)$,
% finally $\lambda(L_{g_i})$ denotes the ordered eigenvalues of the graph Laplacian of $g_i\in G_a$ (similarly for $g_{i'}$).
% and $d_{\max}$ is the largest layer size of ego-graphs in $G$ and $G'$.
\label{theo:main}
\end{theorem}
\begin{proof}
The full proof is detailed in Appendix \ssym \ref{supp:proof}.
\end{proof}

The analysis in Theorem \ref{theo:main} naturally instantiates our insight about the correspondence between structural similarity and GNN transferability. 
%\QZ{transferable conditions: compatible graph encoder $\Psi$, node features $X$ and training objective. Note that the empirical performance difference on $\calL_\Ours$ varies by different choice of GNN $\Psi$ architectures and node features. To provide a critical insight of \Ours, our study of this term lies on the \Ours specific gap $\Delta_{\mathcal{D}}$ on structural difference.}
It allows us to tell how well an \Ours trained on $G_a$ can work on $G_b$ by only checking the local graph Laplacians of $G_a$ and $G_b$ without actually training any model. In particular, we define the \textit{EGI gap} as $\Delta_{\mathcal{D}}$ in Eq.~\ref{eq:gap}, as other term $C$ is the same for different methods using same GNN encoder. It can be computed to bound the transferability of \Ours regarding its loss difference on the source and target graphs.

\begin{remark}
Our view of a graph $G$ as samples of k-hop ego-graphs is important, as it allows us to obtain node-wise characterization of GNN similarly as in \cite{verma2019stability}. It also allows us to set the depth of ego-graphs in the analysis to be the same as the number of GNN layers (k), since the GNN embedding of each node mostly depends on its k-hop ego-graph instead of the whole graph.
\label{remark:view}
\end{remark}

\begin{remark}
For Eq.~\ref{eq:EGI}, Def \ref{def:respecting} ensures the sampling of GNN embedding at a node always corresponds to sampling an ego-graph from $\mathcal{G}$, which reduces to uniformly sampling from $G=\{g_i\}_{i=1}^n$ under the setting of Theorem \ref{theo:main}. Therefore, the requirement of Def \ref{def:respecting} in the context of Theorem \ref{theo:main} guarantees the analysis to be only depending on the structural information of the graph.
\label{remark:feature}
\end{remark}

\xhdr{Supportive observations.} In Table \ref{tab:synthetic}, in the $\bar{d}$ columns, we compute the average structural difference between two Forest-fire graphs ($\Delta_{\mathcal{D}}$(F,F)) and between Barabasi and Forest-fire graphs ($\Delta_{\mathcal{D}}$(B,F)), based on the RHS of Eq.~\ref{eq:theo}. The results validate the topological difference between graphs generated by different random-graph models, while also verifying our view of graph as k-hop ego-graph samples and the way we propose based on it to characterize structural information of graphs.
We further highlight in the $\delta$(acc) columns the accuracy difference between the GNNs transfered from Forest-fire graphs and Barabasi graphs to Forest-fire graphs. Since Forest-fire graphs are more similar to Forest-fire graphs than Barabasi graphs (as verified in the $\Delta_{\mathcal{D}}$ columns), we expect $\delta$(acc.) to be positive and large, indicating more positive transfer between the more similar graphs. Indeed, the behaviors of \Ours align well with the expectation, which indicates its well-understood transferability and the utility of our theoretical analysis.

\xhdr{Use cases of Theorem \ref{theo:main}.} 
Our Theorem \ref{theo:main} naturally allows for two practical use cases among many others: \textit{point-wise pre-judge} and \textit{pair-wise pre-selection} for \Ours pre-training. 
Suppose we have a target graph $G_b$ which does not have sufficient training labels. In the first setting, we have a single source graph $G_a$ which might be useful for pre-training a GNN to be used on $G_b$. The \Ours gap $\Delta_{\mathcal{D}}(G_a, G_b)$ in Eq.~\ref{eq:gap} can then be computed between $G_a$ and $G_b$ to pre-judge whether such transfer learning would be successful before any actual GNN training (\ie, yes if $\Delta_{\mathcal{D}}(G_a, G_b)$ is empirically much smaller than $1.0$; no otherwise). In the second setting, we have two or more source graphs $\{G_a^1, G_a^2, \ldots\}$ which might be useful for pre-training the GNN. The \Ours gap can then be computed between every pair of $G_a^i$ and $G_b$ to pre-select the best source graph (\ie, select the one with the least \Ours gap).

In practice, the computation of eigenvalues on the small ego-graphs can be rather efficient \cite{arora2005fast}, and we do not need to enumerate all pairs of ego-graphs on two compared graphs especially if the graphs are really large (\eg, with more than a thousand nodes). Instead, we can randomly sample pairs of ego-graphs from the two graphs, update the average difference on-the-fly, and stop when it converges. 
Suppose we need to sample $M$ pairs of k-hop ego-graphs to compare two large graphs, and the average size of ego-graphs are $L$, then the overall complexity of computing Eq.~\ref{eq:theo} is $\mathcal{O}(ML^2)$, where $M$ is often less than 1K and $L$ less than 50. In Appendix \ssym \ref{supp:exp_para}, we report the approximated $\Delta_\mathcal{D}$'s \wrt~different sampling frequencies, and they are indeed pretty close to the actual value even with smaller sample frequencies, showing the feasible efficiency of computing $\Delta_\mathcal{D}$ through sampling.

\xhdr{Limitations.} \Ours is designed to account for the structural difference captured by GNNs (\ie, k-hop ego-graphs). The effectiveness of \Ours could be limited if the tasks on target graphs depend on different structural signals. For example, as Eq.~\ref{eq:gap} is computing the average pairwise distances between the graph Laplacians of local ego-graphs, $\Delta_\D$ is possibly less effective in explicitly capturing global graph properties such as numbers of connected components (CCs). In some specific tasks (such as counting CCs or community detection) where such properties become the key factors, $\Delta_\mathcal{D}$ may fail to predict the transferability of GNNs.

\begin{table*}[h]
\small
\begin{center}
\caption{Synthetic experiments of identifying structural equivalent nodes. We randomly generate 40 graphs with the Forest-fire model (F) \cite{leskovec2005graphs} and 40 graphs with the Barabasi model (B) \cite{albert2002statistical}, The GNN model is GIN \cite{xu2019powerful} with random parameters (baseline with only the neighborhood aggregation function), VGAE\cite{kipf2016variational}, DGI \cite{velivckovic2018deep}, and \Ours with GIN encoder. We train VGAE, DGI and \Ours on one graph from either set (F and B), and test them on the rest of Forest-fire graphs (F). Transferable feature is node degree one-hot encoding and non-transferable feature is uniform vectors.
More details about the results and dataset can be found in Appendix \ssym \ref{supp:exp_syn}}.
\label{tab:synthetic}
\scalebox{0.95}{
\begin{tabular}{lll|c|c|c|c|c|c|c|c}
\toprule
\multicolumn{3}{c|}{\multirow{2}{*}{Method}}                             & \multicolumn{3}{c|}{\textbf{transferable features}}            & \multicolumn{3}{c|}{\textbf{non-transferable feature}} & \multicolumn{2}{c}{\textbf{structural difference   }}                                                                                                                   \\
\multicolumn{3}{c|}{}                                                    & F-F & B-F & $\delta$(acc.) & F-F & B-F & $\delta$(acc.)  & $\Delta_{\mathcal{D}}$(F,F) & $\Delta_{\mathcal{D}}$(B,F)\\ \midrule
%\multicolumn{2}{c}{\multirow{3}{*}{}}             & GraphSAGE        &   82.2 $\pm$ 1.1\%      &    N.A.            &  N.A.        &  N.A.     &   83.1 $\pm$ 0.8\%      &    N.A.            &  N.A.        &  N.A.                                    \\
%\midrule
%\multicolumn{2}{l}{}                              % & GCN (untrained)                    &  0.478       &   0.478       &  /       &  0.229 & 0.229 & /  & \multirow{5}{*}{1.78}  &   \multirow{5}{*}{2.17}                             \\
\multicolumn{2}{l}{}                              & GIN (untrained)                    &  0.572       &   0.572       &  /       &  0.358 & 0.358 & / &  \multirow{5}{*}{0.752}  &   \multirow{5}{*}{0.883}                               \\
\multicolumn{2}{l}{}                              & VGAE (GIN)                   &   0.498     &   0.432  &  +0.066       &    0.240      & 0.239  & 0.001 & &                          \\
\multicolumn{2}{l}{}                              & DGI (GIN)                    &  0.578       &    0.591     &  -0.013      & 0.394         & 0.213 &   +0.181 & &                       \\
\multicolumn{2}{l}{}                              & \Ours (GIN)                    & \textbf{0.710}       &    0.616        &   +0.094    &  0.376     & 0.346 &    +0.03    & &                  \\
\bottomrule

\end{tabular}
}
\end{center}
\vspace{-10pt}
\end{table*}

\section{Real Data Experiments}
\label{sec:experiment}
% datasets, protocols, performance, runtimes

\xhdr{Baselines.} We compare the proposed model against existing self-supervised GNNs and pre-training GNN algorithms. To exclude the impact of different GNN encoders $\Psi$ on transferability, we always use the same encoder architecture for all compared methods (\ie, GIN~\cite{xu2019powerful} for direct-transfering experiments, GCN~\cite{kipf2016semi} for transfering with fine-tuning).

The self-supervised GNN baselines are GVAE \cite{kipf2016variational}, DGI \cite{velivckovic2018deep} and two latest mutual information estimation methods GMI~\cite{peng2020graph} and MVC~\cite{hassani2020contrastive}. As for pre-training GNN algorithms, MaskGNN and ContextPredGNN are two node-level pre-training models proposed in~\cite{hu2019strategies}
%\footnote{We are not exploring graph-level tasks but focusing on transfer learning between two graphs. Thus, we drop the graph-level pre-training tasks in the paper since it is not applicable to our setting.}. 
Besides, Structural Pre-train~\cite{hu2019pre} also conducts unsupervised node-level pre-training with structural features like node degrees and clustering coefficients.

\xhdr{Experimental Settings.}  The main hyperparameter $k$ is set 2 in \Ours as a common practice. We use Adam~\cite{kingma2014adam} as optimizer and learning rate is 0.01. We provide the experimental result with varying $k$ in the Appendix \ssym \ref{supp:exp_para}.
All baselines are set with the default parameters. Our experiments were run on an AWS g4dn.2xlarge machine with 1 Nvidia T4 GPU.
By default, we use node degree one-hot encoding as the transferable feature across all different graphs.
As stated before, other transferable features like spectral and other pre-computed node embeddings are also applicable. 
%We mainly want to study and compare the transferability of different GNNs. Thus, different from the semi-supervised learning setiing~\cite{kipf2016semi}, no task-specific labels are used during the GNN training. 
We focus on the setting where the downstream tasks on target graphs are unspecified but assumed to be structure-relevant, and thus pre-train the GNNs on source graphs in an unsupervised fashion.\footnote{The downstream tasks are unspecified because we aim to study the general transferability of GNNs that is not bounded to specific tasks. Nevertheless, we assume the tasks to be relevant to graph structures.} %because otherwise the study is irrelevant to GNNs.}
%In Theorem \ref{theo:main}, we bridge the performance gap between source and target graphs using the proposed \Ours objective. In the LHS of the transferability bound, we derive a rigorous performance gap between source and target graph on \Ours objective. 
In terms of evaluation, 
%we notice that existing pre-training GNNs all assume smaller transfer gap on the pre-training task and further generalization to different tasks. To push forward a fair comparison, we skip the direct evaluation towards the ego-graph reconstruction task which clearly favors our model, but design two more realistic experimental settings: 
we design two realistic experimental settings: 
(1) Direct-transfering  on the more structure-relevant task of role identification without given node features to directly evaluate the utility and transferability of \Ours. 
(2) Few-shot learning on relation prediction with task-specific node features to evaluate the generalization ability of \Ours.

\subsection{Direct-transfering on role identification}
\label{subsec:role}
\label{sec:strcuture-preserving}
%\carl{more details of transfer setting, connections to theoretical conclusions and synthetic exp.}
%In order to carry out un-biased transfer learning evaluation across different methods, the evaluation tasks should be different from any of the pre-training objectives. 
First, we use the role identification without node features in a \textit{direct-transfering} setting as a reliable proxy to evaluate transfer learning performance regarding different pre-training objectives. Role in a network is defined as nodes with similar structural behaviors, such as \textit{clique members}, \textit{hub} and \textit{bridge} \cite{henderson2012rolx}. Across graphs in the same domain, we assume the definition of role to be consistent, and the task of role identification is highly structure-relevant, which can directly reflect the transferability of different methods and allows us to conduct the analysis according to Theorem \ref{theo:main}. Upon convergence of pre-training each model on the source graphs, we directly apply them to the target graphs and further train a multi-layer perceptron (MLP) upon their outputs. The GNN parameters are frozen during the MLP training. We refer to this strategy as \textit{direct-transfering} since there is no fine-tuning of the models after transfering to the target graphs.  

We use two real-world network datasets with role-based node labels: (1) Airport \cite{ribeiro2017struc2vec} contains three networks from different regions-- Brazil, USA and Europe. Each node is an airport and each link is the flight between airports. The airports are assigned with external labels based on their \textit{level of popularity}. (2) Gene \cite{yang2019conditional} contains the gene interactions regarding 50 different cancers. Each gene has a binary label indicating whether it is a \textit{transcription factor}.  More details about the results and dataset  can be found in Appendix \ref{supp:exp_airport}.

The experimental setup on the Airport dataset closely resembles that of our synthetic experiments in Table \ref{tab:synthetic}, but with real data and more detailed comparisons.
We train all models (except for the untrained ones) on the Europe network, and test them on all three networks.
The results are presented in Table~\ref{tab:role-classification}.
We notice that the node degree features themselves (with MLP) show reasonable performance in all three networks, which is not surprising since the popularity-based airport role labels are highly relevant to node degrees. 
The untrained GIN encoder yields a significant margin over just node features, as GNN encoder incorporates structural information to node representations.
While training of the DGI can further improve the performance on the source graph, 
\Ours shows the best performance there with the structure-relevant node degree features, corroborating the claimed effectiveness of \Ours in capturing the essential graph information (\ie recover the k-hop ego-graph distributions) as we stress in \ssym\ref{sec:method}. 

When transfering the models to USA and Brazil networks, \Ours further achieves the best performance compared with all baselines when structure relevant features are used (64.55 and 73.15), which reflects the most significant positive transfer.
Interestingly, direct application of GVAE, DGI and MVC that do not capture the input k-hop graph jointly, leads to rather limited and even negative transferrability (through comparison against the untrained GIN encoders). 
The recently proposed transfer learning frameworks for GNN like MaskGNN and Structural Pre-train are able to mitigate negative transfer to some extent, but their performances are still inferior to \Ours. We believe this is because their models are prone to learn the graph-specific information that is less transferable across different graphs. GMI is also known to capture the graph structure and node features, so it achieves second best result comparing with \Ours. 
%do not aim to capture the underlying ego-graph distributions as we deem important, so they are prune to learn the graph-specific information that is less transferable across different graphs.

Similarly as in Table \ref{tab:synthetic}, we also compute the structural differences among three networks \wrt~the \Ours gap in Eq.~\ref{eq:gap}. The structural difference is 0.869 between the Europe and USA networks, and 0.851 between the Europe and Brazil datasets, which are pretty close. Consequently, the transferability of \Ours regarding its performance gain over the untrained GIN baseline is $4.8\%$ on the USA network and $4.4\%$ on the Brazil network, which are also close. Such observations again align well with our conclusion in Theorem \ref{theo:main} that the transferability of \Ours is closely related to the structural differences between source and target graphs.

%By comparing our results with the row of GIN, we can clearly see positive transfer when using node degree as features (\ie 64.55 vs. 61.56, 73.15 vs. 70.04 on USA and brazil dataset). Since \Ours use the same encoder architecture as the GIN encoder, the performance gain from transfer learning is from the transferability of our ego-graph pre-training objective. Interestingly, we observed most of the baselines like DGI and VAE fails to outperform GIN's performance on target graphs. Existing pre-training GNN frameworks~\cite{hu2019strategies,hu2019pre} . 
%For Structural Pre-trains~\cite{hu2019pre}, we used same feature as proposed rather than node degrees. It shows surprisingly good performance on one target graph (Brazil) instead source graphs, which may credit to the feature's relatedness on Brazil graph. \Ours tends to be a good transfer learning algorithm, that yield good performance when the transferable condition holds (\ie transferable features and similar graphs).
%\QZ{add more explainations when direct transfer results are clear.}
%\carl{also highlight the consistency synthetic exp.}
%Our rigorous transferability bound quantifies performance graph using the structural difference

\begin{table*}[h]
\small
\begin{center}
\caption{Results of role identification with direct-transfering on the Airport dataset. We report mean and standard deviation over 100 runs. The scores marked with $^{**}$ passed t-test with $p < 0.01$ over the second runners.}
\label{tab:role-classification}
\scalebox{0.95}{
\begin{tabular}{lll|c|c|c}
\toprule
\multicolumn{3}{c|}{\multirow{2}{*}{Method}}                             & \multicolumn{3}{c}{\textbf{Airport~\cite{ribeiro2017struc2vec}}}                                          \\
\multicolumn{3}{c|}{}                                                    & Europe & USA & Brazil \\ \midrule

\multicolumn{2}{l}{}                              & features                   &    0.528$\pm$0.052      &   0.557$\pm$0.028	      &  0.671$\pm$0.089     \\
\multicolumn{2}{l}{}                              & GIN (random-init)                    &    0.558$\pm$0.050	     &   0.616$\pm$0.030	       &  0.700$\pm$0.082   \\
\multicolumn{2}{l}{}                              & GVAE (GIN) ~\cite{kipf2016variational}                  & 0.539$\pm$0.053	     &   0.555$\pm$0.029	      &  0.663$\pm$0.089     \\
\multicolumn{2}{l}{}                              & DGI (GIN) ~\cite{velivckovic2018deep}                   &    0.578$\pm$0.050	     &    0.549$\pm$0.028	      &  0.673$\pm$0.084       \\

\multicolumn{2}{l}{}                              & Mask-GIN~\cite{hu2019strategies}                   &    0.564$\pm$0.053	    &  0.608$\pm$0.027	     & 0.667$\pm$0.073  \\
\multicolumn{2}{l}{}                              & ContextPred-GIN~\cite{hu2019strategies}                    &    0.527$\pm$0.048	   &   0.504$\pm$0.030	       &  0.621$\pm$0.078    \\
\multicolumn{2}{l}{}                              & Structural Pre-train~\cite{hu2019pre}                   &   0.560$\pm$0.050	    &   0.622$\pm$0.030	      &    0.688$\pm$0.082 \\

\multicolumn{2}{l}{}                              & MVC~\cite{hassani2020contrastive}                    &   0.532$\pm$0.050	       &  0.597$\pm$0.030	   &  0.661$\pm$0.093                \\

\multicolumn{2}{l}{}                              & GMI~\cite{peng2020graph}                   & 0.581$\pm$0.054	      &  0.593$\pm$0.031	     & 0.731$\pm$0.107
                     \\

\multicolumn{2}{l}{}                              & \Ours (GIN)                    &   \textbf{0.592$\pm$0.046}$^{**}$	     &  \textbf{0.646$\pm$0.029	}$^{**}$      &   \textbf{0.732$\pm$0.078}    \\
\bottomrule

\end{tabular}
}
\end{center}
\end{table*}

On the Gene dataset, with more graphs available, we focus on \Ours to further validate the utility of Eq.~\ref{eq:theo} in Theorem \ref{theo:main}, regarding the connection between the \Ours gap (Eq.~\ref{eq:gap}) and the performance gap (micro-F1) of \Ours on them. Due to severe label imbalance that removes the performance gaps, we only use the seven brain cancer networks that have a more consistent balance of labels.
As shown in Figure~\ref{fig:TCGA}, we train \Ours on one graph and test it on the other graphs. The $x$-axis shows the \Ours gap, and $y$-axis shows the improvement on micro-F1 compared with an untrained GIN. The negative correlation between two quantities is obvious. Specifically, when the structural difference is smaller than 1, positive transfer is observed (upper left area) as the performance of transferred \Ours is better than untrained GIN, and when the structural difference becomes large ($>1$), negative transfer is observed. We also notice a similar graph pattern, \ie single dense cluster, between source graph and positive transferred target graph $G_2$.

\begin{figure*}[h]
    \centering
    \includegraphics[width=1\textwidth]{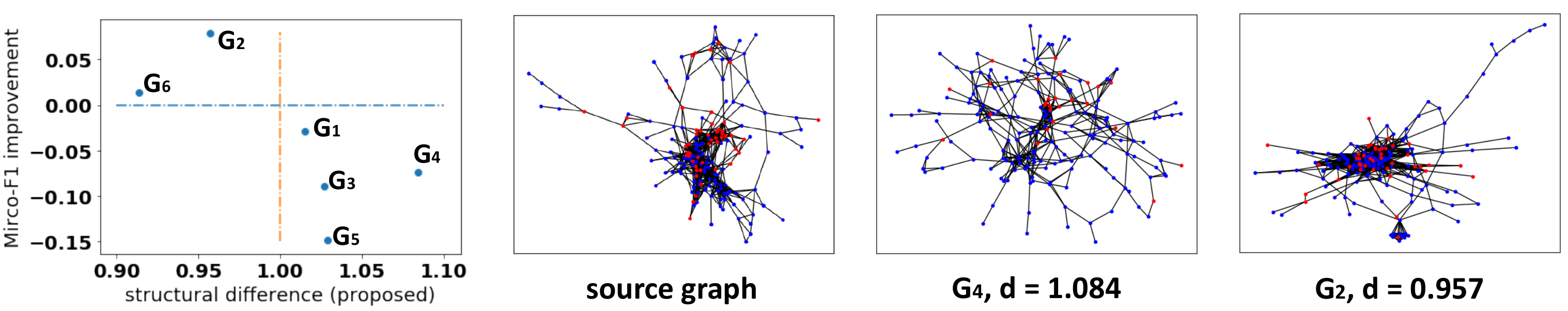}
    \caption{Transfer learning performance of role identification on the Gene dataset. We visualize the source graph $G_0$ and two example target graphs that are relatively more different ($G_4$) or similar ($G_2$) with $G_0$. 
    %More details about the results and dataset can be found in Appendix \ssym3.
    }
    \label{fig:TCGA}
\end{figure*}

\subsection{Few-shot learning on relation prediction}
\label{subsec:relation}
Here we evaluate \Ours in the more generalized and practical setting of \textit{few-shot learning} on the less structure-relevant task of relation prediction, with task-specific node features and fine-tuning.
The source graph contains a cleaned full dump of 579K entities from YAGO \cite{suchanek2007yago}, and we investigate 20-shot relation prediction on a target graph with 24 relation types, which is a sub-graph of 115K entities sampled from the same dump.
%During pre-training, we use the same node degree features for different models. In the fine-tuning phase, we use the positional embedding from LINE~\cite{tang2015line} as the task-specific features for relation prediction. 
In \textit{post-fine-tuning}, the models are pre-trained with an unsupervised loss on the source graph and fine-tuned with the task-specific loss on the target graph.
In \textit{joint-fine-tuning}, the same pre-trained models are jointly optimized \wrt~the unsupervised pre-training loss and task-specific fine-tuning loss on the target graph.
In Table~\ref{tab:finetune-result}, we observe most of the existing models fail to transfer 
%when the node features are different 
across pre-training and fine-tuning tasks, especially in the \textit{joint-fine-tuning} setting. 
In particular, both Mask-GIN and ContextPred-GIN rely a lot on %consistent node features and 
task-specific fine-tuning, while \Ours focuses on the capturing of similar ego-graph structures that are transferable across graphs. The mutual information based method GMI also demonstrates considerable transferability and we believe the ability to capture the graph structure is the key to the transferability.
As a consequence, \Ours significantly outperforms all compared methods in both settings. 
More detailed statistics and running time are in Appendix \ssym \ref{supp:exp_yago}.
\begin{table}[h]
\small
\begin{center}
\caption{Performance of few-shot relation prediction on YAGO. The scores marked with $^{**}$ passed t-test with $p < 0.01$ over the second best results.}
\label{tab:finetune-result}
\scalebox{1.0}{
\begin{tabular}{ll|c|c|c|c}
\toprule
\multicolumn{2}{c|}{\multirow{2}{*}{Method}}                             & \multicolumn{2}{c|}{\textbf{post-fine-tuning}}  & \multicolumn{2}{c}{\textbf{joint-fine-tuning}}                                                                                                                        \\
\multicolumn{2}{c|}{}                                                    & AUROC & MRR & AUROC & MRR  \\ \midrule
%\multicolumn{2}{c}{\multirow{3}{*}{}}             & GraphSAGE        &   82.2 $\pm$ 1.1\%      &    N.A.            &  N.A.        &  N.A.     &   83.1 $\pm$ 0.8\%      &    N.A.            &  N.A.        &  N.A.                                    \\
%\midrule
\multicolumn{2}{l|}{ No pre-train }                      &   0.687$\pm$0.002       &  0.596$\pm$0.003           & N.A. & N.A.                     \\
\multicolumn{2}{l|}{GVAE }   &   0.701$\pm$0.003       &  0.601$\pm$0.007  & 0.679$\pm$0.004 & 0.568$\pm$0.008\\
\multicolumn{2}{l|}{DGI} &    0.689$\pm$0.011       &0.586$\pm$0.025 & 0.688$\pm$0.012 & 0.537$\pm$0.023\\
\multicolumn{2}{l|}{MaskGNN}                     &  0.713$\pm$0.009        &     0.631$\pm$0.015 & 0.712$\pm$0.005 & 0.560$\pm$0.010                     \\
\multicolumn{2}{l|}{ContextPredGNN} &   0.692$\pm$0.030       &  0.662$\pm$0.030  & 0.705$\pm$0.011 & 0.575$\pm$0.021                       \\
\multicolumn{2}{l|}{GMI} &    0.728$\pm$0.005      &  0.625$\pm$0.009  &  0.721$\pm$0.007  & 0.643$\pm$0.011   \\
\multicolumn{2}{l|}{Structural Pre-train} &       OOM   &  OOM  & OOM &  OOM \\
\multicolumn{2}{l|}{MVC} &       OOM   &  OOM  & OOM &  OOM                      \\

\multicolumn{2}{l|}{\Ours}                  &   \textbf{0.739$\pm$ 0.009}$^{**}$        &    \textbf{0.670$\pm$0.014}    &  \textbf{0.787 $\pm$ 0.011}$^{**}$ & \textbf{0.729 $\pm$ 0.016}$^{**}$
% \multicolumn{2}{l}{}                              & GCN (random-init)           & 82.5 $\pm$ 1.2\%         &   74.2 $\pm$ 2.6\%    &  72.0 $\pm$ 2.6\% &  38.4 $\pm$ 23.2\%                 \\

\\
\bottomrule

\end{tabular}
}
\end{center}
\end{table}
\vspace{-6pt}
\section{Conclusion}
\vspace{-6pt}
\label{sec:con}
To the best of our knowledge, this is the first research effort towards establishing a theoretically grounded framework to analyze GNN transferability, which we also demonstrate to be practically useful for guiding the design and conduct of transfer learning with GNNs. For future work, it is intriguing to further strengthen the bound with relaxed assumptions, rigorously extend it to the more complicated and less restricted settings regarding node features and downstream tasks, as well as analyze and improve the proposed framework over more transfer learning scenarios and datasets. It is also important to protect the privacy of pre-training data to avoid potential negative societal impacts.

%Authors are required to include a statement of the broader impact of their work, including its ethical aspects and future societal consequences.
%Authors should discuss both positive and negative outcomes, if any. For instance, authors should discuss a)
%who may benefit from this research, b) who may be put at disadvantage from this research, c) what are the consequences of failure of the system, and d) whether the task/method leverages biases in the data. If authors believe this is not applicable to them, authors can simply state this.
%Use unnumbered first level headings for this section, which should go at the end of the paper. {\bf Note that this section does not count towards the eight pages of content that are allowed.}

\begin{ack}
Research was supported in part by US DARPA KAIROS Program No.~FA8750-19-2-1004, SocialSim Program No.~W911NF-17-C-0099, and INCAS Program No.~HR001121C0165, National Science Foundation IIS-19-56151, IIS-17-41317, and IIS 17-04532, and the Molecule Maker Lab Institute: An AI Research Institutes program supported by NSF under Award No.~2019897. 
Chao Zhang is supported NSF IIS-2008334, IIS-2106961, and ONR MURI N00014-17-1-2656.
We would like to thank AWS Machine Learning Research Awards program for providing computational resources for the experiments in this paper.
This work is also partially supported by the internal funding and GPU servers provided by the Computer Science Department of Emory University. 
Any opinions, findings, and conclusions or recommendations expressed herein are those of the authors and do not necessarily represent the views, either expressed or implied, of DARPA or the U.S.~Government.
\end{ack}

\bibliography{ref}

\begin{thebibliography}{10}

\bibitem{albert2002statistical}
R{\'e}ka Albert and Albert-L{\'a}szl{\'o} Barab{\'a}si.
\newblock Statistical mechanics of complex networks.
\newblock {\em Reviews of modern physics}, 74(1):47, 2002.

\bibitem{arora2005fast}
Sanjeev Arora, Elad Hazan, and Satyen Kale.
\newblock Fast algorithms for approximate semidefinite programming using the
  multiplicative weights update method.
\newblock In {\em FOCS}, pages 339--348, 2005.

\bibitem{baek2020learning}
Jinheon Baek, Dong~Bok Lee, and Sung~Ju Hwang.
\newblock Learning to extrapolate knowledge: Transductive few-shot out-of-graph
  link prediction.
\newblock {\em Advances in Neural Information Processing Systems}, 33, 2020.

\bibitem{bai2016subgraph}
Lu~Bai and Edwin~R Hancock.
\newblock Fast depth-based subgraph kernels for unattributed graphs.
\newblock {\em Pattern Recognition}, 50:233--245, 2016.

\bibitem{barabasi1999emergence}
Albert-L{\'a}szl{\'o} Barab{\'a}si and R{\'e}ka Albert.
\newblock Emergence of scaling in random networks.
\newblock {\em science}, 286(5439):509--512, 1999.

\bibitem{belkin2002laplacian}
Mikhail Belkin and Partha Niyogi.
\newblock Laplacian eigenmaps and spectral techniques for embedding and
  clustering.
\newblock In {\em NIPS}, pages 585--591, 2002.

\bibitem{ben2007analysis}
Shai Ben-David, John Blitzer, Koby Crammer, and Fernando Pereira.
\newblock Analysis of representations for domain adaptation.
\newblock In {\em NIPS}, pages 137--144, 2007.

\bibitem{borgwardt2020graph}
Karsten Borgwardt, Elisabetta Ghisu, Felipe Llinares-L{\'o}pez, Leslie O'Bray,
  and Bastian Rieck.
\newblock Graph kernels: State-of-the-art and future challenges.
\newblock {\em arXiv preprint arXiv:2011.03854}, 2020.

\bibitem{bruna2013spectral}
Joan Bruna, Wojciech Zaremba, Arthur Szlam, and Yann LeCun.
\newblock Spectral networks and locally connected networks on graphs.
\newblock In {\em ICLR}, 2014.

\bibitem{chen2018fastgcn}
Jie Chen, Tengfei Ma, and Cao Xiao.
\newblock Fastgcn: fast learning with graph convolutional networks via
  importance sampling.
\newblock In {\em ICLR}, 2018.

\bibitem{chung1997spectral}
Fan~RK Chung and Fan~Chung Graham.
\newblock {\em Spectral graph theory}.
\newblock Number~92. American Mathematical Soc., 1997.

\bibitem{defferrard2016convolutional}
Micha{\"e}l Defferrard, Xavier Bresson, and Pierre Vandergheynst.
\newblock Convolutional neural networks on graphs with fast localized spectral
  filtering.
\newblock In {\em NIPS}, pages 3844--3852, 2016.

\bibitem{devlin2018bert}
Jacob Devlin, Ming-Wei Chang, Kenton Lee, and Kristina Toutanova.
\newblock Bert: Pre-training of deep bidirectional transformers for language
  understanding.
\newblock In {\em ACL}, pages 4171--4186, 2019.

\bibitem{grover2016node2vec}
Aditya Grover and Jure Leskovec.
\newblock node2vec: Scalable feature learning for networks.
\newblock In {\em KDD}, pages 855--864, 2016.

\bibitem{hamilton2017inductive}
Will Hamilton, Zhitao Ying, and Jure Leskovec.
\newblock Inductive representation learning on large graphs.
\newblock In {\em NIPS}, pages 1024--1034, 2017.

\bibitem{hammond2011wavelets}
David~K Hammond, Pierre Vandergheynst, and R{\'e}mi Gribonval.
\newblock Wavelets on graphs via spectral graph theory.
\newblock {\em ACHA}, 30(2):129--150, 2011.

\bibitem{hassani2020contrastive}
Kaveh Hassani and Amir~Hosein Khasahmadi.
\newblock Contrastive multi-view representation learning on graphs.
\newblock In {\em International Conference on Machine Learning}, pages
  4116--4126. PMLR, 2020.

\bibitem{he2016deep}
Kaiming He, Xiangyu Zhang, Shaoqing Ren, and Jian Sun.
\newblock Deep residual learning for image recognition.
\newblock In {\em CVPR}, pages 770--778, 2016.

\bibitem{henderson2012rolx}
Keith Henderson, Brian Gallagher, Tina Eliassi-Rad, Hanghang Tong, Sugato Basu,
  Leman Akoglu, Danai Koutra, Christos Faloutsos, and Lei Li.
\newblock Rolx: structural role extraction \& mining in large graphs.
\newblock In {\em KDD}, pages 1231--1239, 2012.

\bibitem{hjelm2018learning}
R~Devon Hjelm, Alex Fedorov, Samuel Lavoie-Marchildon, Karan Grewal, Phil
  Bachman, Adam Trischler, and Yoshua Bengio.
\newblock Learning deep representations by mutual information estimation and
  maximization.
\newblock In {\em ICLR}, 2019.

\bibitem{hu2019strategies}
Weihua Hu, Bowen Liu, Joseph Gomes, Marinka Zitnik, Percy Liang, Vijay Pande,
  and Jure Leskovec.
\newblock Strategies for pre-training graph neural networks.
\newblock In {\em ICLR}, 2019.

\bibitem{hu2020gpt}
Ziniu Hu, Yuxiao Dong, Kuansan Wang, Kai-Wei Chang, and Yizhou Sun.
\newblock Gpt-gnn: Generative pre-training of graph neural networks.
\newblock In {\em KDD}, pages 1857--1867, 2020.

\bibitem{hu2019pre}
Ziniu Hu, Changjun Fan, Ting Chen, Kai-Wei Chang, and Yizhou Sun.
\newblock Pre-training graph neural networks for generic structural feature
  extraction.
\newblock {\em arXiv preprint arXiv:1905.13728}, 2019.

\bibitem{hwang2004cauchy}
Suk-Geun Hwang.
\newblock Cauchy's interlace theorem for eigenvalues of hermitian matrices.
\newblock {\em The American Mathematical Monthly}, 111(2):157--159, 2004.

\bibitem{kan2021zero}
Xuan Kan, Hejie Cui, and Carl Yang.
\newblock Zero-shot scene graph relation prediction through commonsense
  knowledge integration.
\newblock In {\em ECML-PKDD}, 2021.

\bibitem{keriven2019universal}
Nicolas Keriven and Gabriel Peyr{\'e}.
\newblock Universal invariant and equivariant graph neural networks.
\newblock In {\em NIPS}, pages 7090--7099, 2019.

\bibitem{kingma2014adam}
Diederik~P Kingma and Jimmy Ba.
\newblock Adam: A method for stochastic optimization.
\newblock {\em arXiv preprint arXiv:1412.6980}, 2014.

\bibitem{kipf2016variational}
Thomas~N Kipf and Max Welling.
\newblock Variational graph auto-encoders.
\newblock {\em arXiv preprint arXiv:1611.07308}, 2016.

\bibitem{kipf2016semi}
Thomas~N Kipf and Max Welling.
\newblock Semi-supervised classification with graph convolutional networks.
\newblock In {\em ICLR}, 2017.

\bibitem{kriege2020survey}
Nils~M Kriege, Fredrik~D Johansson, and Christopher Morris.
\newblock A survey on graph kernels.
\newblock {\em Applied Network Science}, 5(1):1--42, 2020.

\bibitem{lan2020node}
Lin Lan, Pinghui Wang, Xuefeng Du, Kaikai Song, Jing Tao, and Xiaohong Guan.
\newblock Node classification on graphs with few-shot novel labels via meta
  transformed network embedding.
\newblock {\em Advances in Neural Information Processing Systems}, 33, 2020.

\bibitem{leskovec2005graphs}
Jure Leskovec, Jon Kleinberg, and Christos Faloutsos.
\newblock Graphs over time: densification laws, shrinking diameters and
  possible explanations.
\newblock In {\em Proceedings of the eleventh ACM SIGKDD international
  conference on Knowledge discovery in data mining}, pages 177--187, 2005.

\bibitem{levie2019transferability}
Ron Levie, Wei Huang, Lorenzo Bucci, Michael~M Bronstein, and Gitta Kutyniok.
\newblock Transferability of spectral graph convolutional neural networks.
\newblock {\em arXiv preprint arXiv:1907.12972}, 2019.

\bibitem{levie2019transferability2}
Ron Levie, Elvin Isufi, and Gitta Kutyniok.
\newblock On the transferability of spectral graph filters.
\newblock In {\em 2019 13th International conference on Sampling Theory and
  Applications (SampTA)}, pages 1--5. IEEE, 2019.

\bibitem{liu2019graph}
Jenny Liu, Aviral Kumar, Jimmy Ba, Jamie Kiros, and Kevin Swersky.
\newblock Graph normalizing flows.
\newblock In {\em Advances in Neural Information Processing Systems}, pages
  13556--13566, 2019.

\bibitem{mcpherson2001birds}
Miller McPherson, Lynn Smith-Lovin, and James~M Cook.
\newblock Birds of a feather: Homophily in social networks.
\newblock {\em Annual review of sociology}, 27(1):415--444, 2001.

\bibitem{mikolov2013distributed}
Tomas Mikolov, Ilya Sutskever, Kai Chen, Greg Corrado, and Jeffrey Dean.
\newblock Distributed representations of words and phrases and their
  compositionality.
\newblock {\em arXiv preprint arXiv:1310.4546}, 2013.

\bibitem{nikolentzos2019graph}
Giannis Nikolentzos, Giannis Siglidis, and Michalis Vazirgiannis.
\newblock Graph kernels: A survey.
\newblock {\em arXiv preprint arXiv:1904.12218}, 2019.

\bibitem{oono2020graph}
Kenta Oono and Taiji Suzuki.
\newblock Graph neural networks exponentially lose expressive power for node
  classification.
\newblock In {\em ICLR}, 2020.

\bibitem{page1999pagerank}
Lawrence Page, Sergey Brin, Rajeev Motwani, and Terry Winograd.
\newblock The pagerank citation ranking: Bringing order to the web.
\newblock Technical report, Stanford InfoLab, 1999.

\bibitem{peng2020graph}
Zhen Peng, Wenbing Huang, Minnan Luo, Qinghua Zheng, Yu~Rong, Tingyang Xu, and
  Junzhou Huang.
\newblock Graph representation learning via graphical mutual information
  maximization.
\newblock In {\em WWW}, pages 259--270, 2020.

\bibitem{perozzi2014deepwalk}
Bryan Perozzi, Rami Al-Rfou, and Steven Skiena.
\newblock Deepwalk: Online learning of social representations.
\newblock In {\em KDD}, pages 701--710, 2014.

\bibitem{qiu2020gcc}
Jiezhong Qiu, Qibin Chen, Yuxiao Dong, Jing Zhang, Hongxia Yang, Ming Ding,
  Kuansan Wang, and Jie Tang.
\newblock Gcc: Graph contrastive coding for graph neural network pre-training.
\newblock In {\em KDD}, pages 1150--1160, 2020.

\bibitem{ravi2016optimization}
Sachin Ravi and Hugo Larochelle.
\newblock Optimization as a model for few-shot learning.
\newblock In {\em ICLR}, 2017.

\bibitem{ribeiro2017struc2vec}
Leonardo~FR Ribeiro, Pedro~HP Saverese, and Daniel~R Figueiredo.
\newblock struc2vec: Learning node representations from structural identity.
\newblock In {\em KDD}, pages 385--394, 2017.

\bibitem{roweis2000nonlinear}
Sam~T Roweis and Lawrence~K Saul.
\newblock Nonlinear dimensionality reduction by locally linear embedding.
\newblock {\em Science}, 290(5500):2323--2326, 2000.

\bibitem{ruiz2020graphon}
Luana Ruiz, Luiz Chamon, and Alejandro Ribeiro.
\newblock Graphon neural networks and the transferability of graph neural
  networks.
\newblock {\em Advances in Neural Information Processing Systems}, 33, 2020.

\bibitem{shi2018easing}
Yu~Shi, Qi~Zhu, Fang Guo, Chao Zhang, and Jiawei Han.
\newblock Easing embedding learning by comprehensive transcription of
  heterogeneous information networks.
\newblock In {\em Proceedings of the 24th ACM SIGKDD International Conference
  on Knowledge Discovery \& Data Mining}, pages 2190--2199, 2018.

\bibitem{suchanek2007yago}
Fabian~M Suchanek, Gjergji Kasneci, and Gerhard Weikum.
\newblock Yago: a core of semantic knowledge.
\newblock In {\em WWW}, pages 697--706, 2007.

\bibitem{sun2019infograph}
Fan-Yun Sun, Jordan Hoffman, Vikas Verma, and Jian Tang.
\newblock Infograph: Unsupervised and semi-supervised graph-level
  representation learning via mutual information maximization.
\newblock In {\em ICLR}, 2019.

\bibitem{tang2015line}
Jian Tang, Meng Qu, Mingzhe Wang, Ming Zhang, Jun Yan, and Qiaozhu Mei.
\newblock Line: Large-scale information network embedding.
\newblock In {\em WWW}, pages 1067--1077, 2015.

\bibitem{tenenbaum2000global}
Joshua~B Tenenbaum, Vin De~Silva, and John~C Langford.
\newblock A global geometric framework for nonlinear dimensionality reduction.
\newblock {\em Science}, 290(5500):2319--2323, 2000.

\bibitem{velivckovic2017graph}
Petar Velickovic, Guillem Cucurull, Arantxa Casanova, Adriana Romero, Pietro
  Lio, and Yoshua Bengio.
\newblock Graph attention networks.
\newblock In {\em ICLR}, 2018.

\bibitem{velivckovic2018deep}
Petar Velickovic, William Fedus, William~L Hamilton, Pietro Lio, Yoshua Bengio,
  and R~Devon Hjelm.
\newblock Deep graph infomax.
\newblock In {\em ICLR}, 2019.

\bibitem{verma2019stability}
Saurabh Verma and Zhi-Li Zhang.
\newblock Stability and generalization of graph convolutional neural networks.
\newblock In {\em KDD}, 2019.

\bibitem{vinyals2016matching}
Oriol Vinyals, Charles Blundell, Tim Lillicrap, Daan Wierstra, et~al.
\newblock Matching networks for one shot learning.
\newblock In {\em NIPS}, pages 3630--3638, 2016.

\bibitem{vishwanathan2010graph}
S~Vichy~N Vishwanathan, Nicol~N Schraudolph, Risi Kondor, and Karsten~M
  Borgwardt.
\newblock Graph kernels.
\newblock {\em Journal of Machine Learning Research}, 11:1201--1242, 2010.

\bibitem{weisfeiler1968reduction}
Boris Weisfeiler and Andrei~A Lehman.
\newblock A reduction of a graph to a canonical form and an algebra arising
  during this reduction.
\newblock {\em Nauchno-Technicheskaya Informatsia}, 2(9):12--16, 1968.

\bibitem{wu2020unsupervised}
Man Wu, Shirui Pan, Chuan Zhou, Xiaojun Chang, and Xingquan Zhu.
\newblock Unsupervised domain adaptive graph convolutional networks.
\newblock In {\em WWW}, pages 1457--1467, 2020.

\bibitem{xu2019powerful}
Keyulu Xu, Weihua Hu, Jure Leskovec, and Stefanie Jegelka.
\newblock How powerful are graph neural networks?
\newblock In {\em ICLR}, 2019.

\bibitem{yang2014embedding}
Bishan Yang, Wen-tau Yih, Xiaodong He, Jianfeng Gao, and Li~Deng.
\newblock Embedding entities and relations for learning and inference in
  knowledge bases.
\newblock {\em arXiv preprint arXiv:1412.6575}, 2014.

\bibitem{yang2018meta}
Carl Yang, Yichen Feng, Pan Li, Yu~Shi, and Jiawei Han.
\newblock Meta-graph based hin spectral embedding: Methods, analyses, and
  insights.
\newblock In {\em ICDM}, 2018.

\bibitem{yang2020multisage}
Carl Yang, Aditya Pal, Andrew Zhai, Nikil Pancha, Jiawei Han, Chuck Rosenberg,
  and Jure Leskovec.
\newblock Multisage: Empowering graphsage with contextualized multi-embedding
  on web-scale multipartite networks.
\newblock In {\em KDD}, 2020.

\bibitem{yang2020heterogeneous}
Carl Yang, Yuxin Xiao, Yu~Zhang, Yizhou Sun, and Jiawei Han.
\newblock Heterogeneous network representation learning: A unified framework
  with survey and benchmark.
\newblock In {\em TKDE}, 2020.

\bibitem{yang2018did}
Carl Yang, Chao Zhang, Xuewen Chen, Jieping Ye, and Jiawei Han.
\newblock Did you enjoy the ride? understanding passenger experience via
  heterogeneous network embedding.
\newblock In {\em ICDE}, 2018.

\bibitem{yang2020co}
Carl Yang, Jieyu Zhang, and Jiawei Han.
\newblock Co-embedding network nodes and hierarchical labels with taxonomy
  based generative adversarial nets.
\newblock In {\em ICDM}, 2020.

\bibitem{yang2020relation}
Carl Yang, Jieyu Zhang, Haonan Wang, Sha Li, Myungwan Kim, Matt Walker, Yiou
  Xiao, and Jiawei Han.
\newblock Relation learning on social networks with multi-modal graph edge
  variational autoencoders.
\newblock In {\em WSDM}, 2020.

\bibitem{yang2019conditional}
Carl Yang, Peiye Zhuang, Wenhan Shi, Alan Luu, and Pan Li.
\newblock Conditional structure generation through graph variational generative
  adversarial nets.
\newblock In {\em NIPS}, pages 1338--1349, 2019.

\bibitem{ying2018hierarchical}
Zhitao Ying, Jiaxuan You, Christopher Morris, Xiang Ren, Will Hamilton, and
  Jure Leskovec.
\newblock Hierarchical graph representation learning with differentiable
  pooling.
\newblock In {\em NIPS}, 2018.

\bibitem{you2018graphrnn}
Jiaxuan You, Rex Ying, Xiang Ren, William Hamilton, and Jure Leskovec.
\newblock {G}raph{RNN}: Generating realistic graphs with deep auto-regressive
  models.
\newblock In {\em Proceedings of the 35th International Conference on Machine
  Learning}, pages 5708--5717. PMLR, 2018.

\bibitem{zhu2020shift}
Qi~Zhu, Natalia Ponomareva, Jiawei Han, and Bryan Perozzi.
\newblock Shift-robust gnns: Overcoming the limitations of localized graph
  training data.
\newblock In {\em NeurIPS}, 2021.

\end{thebibliography}
\bibliographystyle{plain}

\newpage

\appendix

\section{Theory Details}
\label{supp:proof}
From the $\calL_{\textsc{Egi}}$ objective, we have assumed  $g_i\overset{\iid}{\sim} \mu$, $x_i\overset{\iid}{\sim} \nu$, and $(g_i, x_i)\overset{\iid}{\sim} p$, for $(g_i, x_i)\in\mathcal{G}\times \calX$. Then for a sample $\{(g_i, x_i)\}_i$, we have access to the empirical distributions of the three. In the procedure of evaluating the objective, we sample uniformly.

Note that, in Eq.~2 of the main paper, we used a $d$ dimensional hidden state $h_{p}$ to denote an edge's source node representation and $x_q$ as destination node features from the structure of the ego-graph and the associated source node feature with GNN. In our proof, we denote $v_{p,q}$ as the $q$-th node in the $p$-th layer of the ego-graph and let $h_{p,q} = h_p$ and $x_{p,q}=x_q$.
For simplicity, in i-th layer, we denote $f(x^i)=h_{p,q}^{i}\|x^i_{p,q}$, where $[\cdot \| \cdot]$ is the concatenation operation.

Finally, as we are considering GNN with $k$ layers, its computation only depends on the k-hop ego-graphs of $G$, which is an important consideration when unfolding the embedding of GNN at a centre node with Lamma \ref{lemma:OptNorm}.

\begin{lemma}{}
\label{lemma:OptNorm}
For any $A\in \RR^{m\times n}$, where $m\geq n$, and $A$ is a submatrix of $B\in \RR^{m'\times n}$, where $m<m'$, we have
$$
\|A\|_2\leq \|B\|_2.
$$
\end{lemma}
\begin{proof}
Note that, $AA^T$ is a principle matrix of $BB^T$, \ie, $AA^T$ is obtained by removing the same set of rows and columns from $BB^T$. Then, by Eigenvalue Interlacing Theorem \cite{hwang2004cauchy} and the fact that $A^TA$ and $AA^T$ have the same set of non-zero singular values, the matrix operator norm satisfies
$\|A\|_2=\sqrt{\lambda_{\max}(A^TA)} = \sqrt{\lambda_{\max}(AA^T)}\leq 
\sqrt{\lambda_{\max}(BB^T)} =
\|B\|_2$.
\end{proof}

\subsection{Center-node view of GCN}
\label{Subsec:view}
Recall that $V_p(g_i)$ denotes the set of nodes in the $p$th hop of k-hop ego-graph $g_i$, and $x_{p,q}^i$ denotes the feature for $q$th node in $p$th hop of $g_i$, for any $p=0,\dots, k;\; q=1, \dots, |V_p(g_i)|$. Similarly, $V(g_i)$ denotes the entire set of nodes in $g_i$. In each ego-graph sample $\{g_i, x_i\}$, the layer-wise propagation rules for the center node embedding in encoder $\Psi$ and discriminator $\D$ can be written into the form of GCN as followed
$$
Z^{(l)} = \text{ReLU} (D^{-\frac{1}{2}}(I+A)D^{-\frac{1}{2}}Z^{(l-1)} \theta^{(l)} )
%, \ h_i^{(l)} = \text{ReLU} (\Tilde{ D}^{-\frac{1}{2}}(\Tilde{ I}+\Tilde{ A})\Tilde{ D}^{-\frac{1}{2}}h_i^{(l-1)} \Tilde{\theta}^{(l)} )
$$
where $A$ is adjacency matrix of $G$. $I$ adds the self-loop and $D_{ii} = \sum_j A_{ij}$ is the degree matrix. 

We focus on the center node's embedding obtained from a $k$-layer GCN with 1-hop polynomial filter $\phi(L) = Id - L$. 
Inspired by the characterization of GCN from a node-wise view in \cite{verma2019stability}, we similarly denote the embedding of node $x_i\;\forall i=1,\cdots, n$ in the final layer of the GCN as 
$$
z_i^{(k)} = z_i = \Psi_{\theta}(x_i)=\sigma(\sum_{j\in\N(x_i)} 
e_{i j}{z_j^{(k-1)}}^T\theta^{(k)})\in \RR^d,
$$
where $e_{i j}=[\phi(L)]_{i j}\in\RR$ the weighted link between node $i$ and $j$; and $\theta^{(k)}\in\RR^{d\times d}$ is the weight for the $k$th layer sharing across nodes. Then $\theta=\{\theta^{(\ell)}\}_{\ell=1}^k$. We may denote $z_i^{(\ell)}\in \RR^d$ similarly for $\ell=1,\cdots, k-1$, and  $z_i^{0}=x_i\in\RR^{d}$ as the node feature of center node $x_i$. 
With the assumption of GCN in the statement, it is clear that only the k-hop ego-graph $\gi$ centered at $x_i$ is needed to compute $z_i^{(k)}$ for any $i=1,\cdots, n$ instead of the whole of $G$. Precisely, $p$-hop of subgraph corresponds to the $\ell=(k-p)$th layer in the model. 

With such observation in mind,
let us denote the matrix of node embeddings of $g_i$ at the $\ell$th layer as $[z_{i}^{(\ell)}]\in\RR^{|V(g_i)|\times d}$ for $\ell=1,\cdots, k$; and let $[z_{i}^{(0)}] \equiv [x_{i}]\in (\RR^{d})^{|V(g_i)|}$ denote the matrix of node features in the $k$-hop ego-graph $g_i$.
In addition, denote $[z_{i}^{(\ell)}]_{p}$ as the principle submatrix, which includes embeddings for nodes in the $0$ to $p$th hop of $g_i$, $0\leq p \leq k$.
% ; similarly for $\phi(L_{\gi})_{t}$.

We denote $L_{g_i}$ as the out-degree normalised graph Laplacian of $g_i$. 
 Here, the out-degree is defined with respect to the direction from leaves to centre node in $g_i$. Similarly, denote $\Tilde{L}_{g_i}$ as the in-degree normalised graph Laplacian of $g_i$, where the direction is from centre to leaves.
 
WLOG, we write the $\ell$th layer embedding in matrix notation of the following form
$$
[z_i^{(\ell)}]_{k-\ell+1} = \sigma([\phi(L_{g_i})]_{k-\ell+1}[z_i^{(\ell-1)}]_{k-\ell+1}\theta^{(\ell)}),
$$
where the GCN only updates the embedding of nodes in the $0$ to $(k-\ell)$th hop. We also implicitly assume the embedding of nodes in $(k-\ell+1)$ to $k$th hop are unchanged through the update, due to the directed nature of $g_i$. Hence, we obtain $z_i \equiv [z_i^{(k)}]_0$ from the following
$$
[z_i^{(k)}]_1 = \sigma([\phi(L_{g_i})]_{1}[z_i^{(k-1)}]_{1}\theta^{(k)}).
$$

Similarly, we are able to write down the form of discriminator using matrix representation for GCN. The edge information at $\ell$th time point for nodes in $V(g_i)$ can be described as follows
$$
[h_i^{(\ell)}] = ReLU(\phi(\Tilde{L}_{g_i})[h_i^{(\ell-1)}]\Tilde{\theta}^{(\ell)}),
$$
% Note that, we assume $\phi(L)=Id-L$ for both encoder and decoder in the implementation.

\subsection{Proof for Theorem 4.1}
 We restate Theorem 4.1 from the main paper as below.
\begin{theorem}
Let $G_a=\{(g_i, x_i)\}_{i=1}^n$ and $G_b=\{(\gid, x_{i'})\}_{i'=1}^m$ be two graphs and node features are structure-respecting with $x_i = f(L_{g_i}), x_{i'}=f(L_{g_{i'}})$ for some function $f:\RR^{|V(g_i)|\times |V(g_i)|} \to \RR^{d}$.
Consider GCN $\Psi_{\theta}$ with k layers and a 1-hop polynomial filter $\phi$,the empirical performance difference of $\Psi_{\theta}$ with $\calL_{\Ours}$ satisfies
\begin{align}
|\calL_{\Ours}(G_a) - \calL_{\Ours}(G_b)|\leq \mathcal{O}
\left(\frac{1}{nm}
    \sum_{i=1}^{n}\sum_{i'=1}^{m}
    [M + C
    \lambda_{\max}(L_{\gi} 
    - L_{\gid})
    +
    \Tilde{C}
    \lambda_{\max}(\Tilde{L}_{g_i} 
    - \Tilde{L}_{g_{i'}}))]
\right),
\end{align}
where $M$ is dependant on $\Psi$, $\mathcal{D}$, node features, and the largest eigenvalue of $L_{\gi}$ and $\Tilde{L}_{\gi}$. $C$ is a constant dependant on the encoder, while $\Tilde{C}$ is a constant dependant on the decoder. With a slight abuse of notation, we denote $\lambda_{\max}(A):=\lambda_{\max}(A^TA)^{1/2}$. Note that, in the main paper, we have $C:= M+C\lambda_{\max}(L_{\gi} - L_{\gid})$, and $\Delta_{\D}(G_a, G_b) := \Tilde{C}
    \lambda_{\max}(\Tilde{L}_{g_i} 
    - \Tilde{L}_{g_{i'}})$.
\end{theorem}

\begin{proof}
% We denote $\sigma_s(t) = \log (1+e^t)$, the softplus activation function, which is $1$-Lipschitz continuous. 
Now,
\begin{equation*}
    \begin{aligned}
    &|\calL_{\Ours}(G) - \calL_{\Ours}(G')|\\ 
    =
    &\left|\frac{1}{n^2}
    \sum_{i,j=1}^{n}
    (\D(g_i, z_{j})) - 
    \frac{1}{n}\sum_{i=1}^n 
    (-(-\D(g_i, z_{i})) - (\frac{1}{m^2}\sum_{i',j'=1}^{m}
    (\D(g_{i'}, z_{j'})) - \frac{1}{m}\sum_{i'=1}^m 
    (-(-\D(g_{i'}, z_{i'}))))\right|\\
    \leq &\frac{1}{n^2m^2}
    \sum_{i,j=1}^{n}\sum_{i',j'=1}^{m}
    \left|
    \D(g_i, z_{j}) -  \D(g_{i'}, z_{j'})
    \right| + 
    \frac{1}{nm}
    \sum_{i=1}^{n}\sum_{i'=1}^{m} 
    \left|
    \D(g_i, z_{i}) -\D(g_{i'}, z_{i'})
    \right| \\
    = &\frac{1}{n^2m^2}
    \sum_{i,j=1}^{n}\sum_{i',j'=1}^{m} A +
    \frac{1}{nm}
    \sum_{i=1}^{n}\sum_{i'=1}^{m} B.
    \end{aligned}
\end{equation*}

We make the following assumptions in the proof,
\begin{enumerate}
    \item Assume the size of the neighborhood for each node is bounded by $0<r<\infty$, then the maximum number of node for $p$-th layer subgraph is bounded by $r^p$. WLOG, let $1\leq |V_p(g_i)| \leq  |V_p(g_{i'})| \leq r^p$;
    \item Assume $h_{p,q}^{i}\|x^i_{p,q}=0$ if $|V_p(g_i)|<q$, i.e. assume non-informative edge information and node features for non-existed nodes in the smaller neighborhood with no links;
\end{enumerate}
From Assumption 2, we add isolated nodes to the smaller neighborhood $V_p(g_i)$ such that the neighborhood size at each hop match. It can be found in our code to compute \Ours gap as pad\_nbhd.
For the following proof, we WLOG assume $|V_p(g_i)|=|V_p(g_{i'})|\;\forall p$. 

First we consider $B$. Recall that, $V_p(g_i)$ is the set of nodes in layer $p$ of $g_i$,
$$\D (g_i, z_i) = \Sum_{p=1}^k \Sum_{q=1}^{|V_p(g_i)|} \log( \sigma_{sig} \left( U^T \tau \left( W^T [f(x^i) \| z_i] \right) \right)),$$
where $\sigma_{sig}(t)=\frac{1}{1+e^{-t}}$ is the sigmoid function, $\tau$ is some $\gamma_{\tau}$-Lipschitz activation function and $[\cdot \| \cdot]$ denotes the concatenation of two vectors. Then we obtain
$$
U^T \tau \left( W^T [f(x^i) \| z_i]\right) = 
U^T \tau\left(W_1^T f(x^i) + W_2^T z_i\right).
$$
Since $\log(\sigma_{sig}(t))=-\log(1+e^{-t})$, which is $1$-Lipschitz, it gives
\begin{equation}
\label{Eq:B}
    \begin{aligned}
    B &\leq \sum_p^k|
    \sum_q^{|V_p(g_{i'})|} \sig_s(U^T \tau\left(W_1^T f(x^i) + W_2^T z_i\right)) - \sig_s(U^T \tau\left(W_1^T f(x^{i'}) + W_2^T z_{i'}\right))
    |\\
    % &+ (|V_p(g_{i'})| - |V_p(g_{i})|)|
    % \sigma_s(-U^T \tau\left(W_2^T z_i\right))|\\
    % &\leq \sum_p^k \sum_q^{|V_p(g_{i'})|} 
    % |U^T \tau\left(W_1^T f(x^i) + W_2^T z_i\right) - 
    % U^T \tau\left(W_1^T f(x^{i'}) + W_2^T z_{i'}\right)|\\
    &\leq \gamma_{\tau}\|U\|_2\sum_{p=1}^k 
    \sum_{q=1}^{|V_p(g_{i'})|}(\|W_1^T f(x^i) - W_1^T f(x^{i'})\|_2 
    + \|W_2^T z_i - W_2^T z_{i'}\|_2)
    \\
    &\leq \gamma_{\tau}\|U\|_2 s_w
    \left(
     \sum_{p=1}^k 
    \sum_{q=1}^{|V_p(g_{i'})|}\left[
    \|h^i_{p,q} - h^{i'}_{p,q}\|_2
    + \|x^i_{p,q} - x^{i'}_{p,q}\|_2
    \right] +  \sum_{p=1}^k 
    \sum_{q=1}^{|V_p(g_{i'})|}\|z_i - z_{i'}\|_2
    \right)\\
    &\leq C_1\left(
    \sum_{p=1}^k 
    \sum_{q=1}^{|V_p(g_{i'})|}\left[
    \|h^i_{p,q} - h^{i'}_{p,q}\|_2
    + \|x^i_{p,q} - x^{i'}_{p,q}\|_2
    \right]/\sum_{p=1}^k r^p
    +  \|z_i - z_{i'}\|_2 
    \right)\\
    &=C_1\left(
    I_1 + I_2
    \right)
    \end{aligned}
\end{equation}

We provide the derivation for the unfolding of $\ell$th layer GCN with the centre-node view in Lemma \ref{lemma:z_expansion}. This will be used in the derivation of $I_1$ and $I_2$.
\begin{lemma}
\label{lemma:z_expansion}
For any $\ell=1, \cdots, k$, we have an upper bound for the hidden representation difference between $g_i$ and $g_i^\prime$,
\begin{equation}
    \begin{aligned}
    \|[z_{i}^{(\ell)}]_{k-\ell} - [z_{i'}^{(\ell)}]_{k-\ell}\|_2
    &\leq (\gams c_{\theta})^{\ell}\|\phi(L_{\gi})\|_2^{\ell}
    \|[x_{i}] - [x_{i'}]\|_2\\ &+\frac{(\gams c_{\theta})^{\ell}\|\phi(L_{\gi})\|_2^{\ell}+1}
    {\gams c_{\theta}\|\phi(L_{\gi})\|_2-1}
    \gams c_{\theta}c_z
    \|\phi(L_{\gi}) - \phi(L_{\gid})\|_2.
    \end{aligned}
    \label{equ:z_expansion}
\end{equation}
Specifically, for $\ell=k$, we obtain the expansion for center node embedding $\|[z_{i}^{(k)}]_0 - [z_{i'}^{(k)}]_0\|\equiv \|z_{i} - z_{i'}\|$.
\end{lemma}
\begin{proof}
By Lemma \ref{lemma:OptNorm}, for any $\ell=1,\cdots, k$, the following holds
$$
\|[z_{i}^{(\ell)}]_{k-\ell} - [z_{i'}^{(\ell)}]_{k-\ell}\|_2
    \leq \|[z_{i}^{(\ell)}]_{k-\ell+1} - [z_{i'}^{(\ell)}]_{k-\ell+1}\|_2.
$$
Assume $\max_{\ell}\|[z_{i}^{(\ell)}]\|_2\leq c_z<\infty\;\forall i$, and $\max_{\ell} \|\theta^{(\ell)}\|_2\leq c_{\theta}<\infty$, where $c_{\theta} = \vee_{\ell} s_{\theta^{(\ell)}}$ the largest singular value. 

 Then, for $\ell=1,\cdots, k-1$, we have
%  t=k-l
\begin{equation}
    \begin{aligned}
    &\|[z_{i'}^{(\ell)}]_{k-\ell} - [z_{i'}^{(\ell)}]_{k-\ell}\|_2 \\
    \leq& \|[\sig([\phi(L_{\gi})]_{k-\ell+1}
    [z_{i}^{(\ell-1)}]_{k-\ell+1}\theta^{(\ell)})
    - \sig([\phi(L_{\gid})]_{k-\ell+1}
    [z_{i'}^{(\ell-1)}]_{k-\ell+1}
    \theta^{(\ell)})]_{k-\ell})\|_2\\
    \leq& \gams\|[\phi(L_{\gi})]_{k-\ell+1}
    [z_{i}^{(\ell-1)}]_{k-\ell+1}
    - [\phi(L_{\gid})]_{k-\ell+1}
    [z_{i'}^{(\ell-1)}]_{k-\ell+1}\|_2
    \|\theta^{(k)}\|_2\\
    \leq& \gams c_{\theta}
    \|[\phi(L_{\gi})]_{k-\ell+1}
    \|_2\|[z_{i}^{(\ell-1)}]_{k-\ell+1}
    - [z_{i'}^{(\ell-1)}]_{k-\ell+1}\|_2
    +\gams c_{\theta}
    \|[z_{i'}^{(\ell-1)}]_{k-\ell+1}
    \|_2
    \|[\phi(L_{\gi})]_{k-\ell+1}
    - [\phi(L_{\gid})]_{k-\ell+1}\|_2\\
    % \leq &\gams c_{\theta}\|\phi(L_{\gi})_{k-\ell+1}
    % (z_{i}^{(\ell-1)})_{k-\ell+1}
    % - \phi(L_{\gid})_{k-\ell+1}
    % (z_{i'}^{(\ell-1)})_{k-\ell+1}\|_2\\
    \leq &\gams c_{\theta}\|\phi(L_{\gi})\|_2
    \|[z_{i}^{(\ell-1)}]_{k-\ell+1}
    - [z_{i'}^{(\ell-1)}]_{k-\ell+1}\|_2 +\gams c_{\theta}
    c_z\|\phi(L_{\gi}) - \phi(L_{\gid})\|_2.
    % &\leq Mc_x + c_x\gamma_{\theta}\|L_{\gi} - L_{\gid}\|_2
    \end{aligned}
\label{Eq:l-zi_diff}
\end{equation}
since $[\phi(L_{\gi})]_{k-\ell+1}$ is the principle submatrix of $\phi(L_{\gi})$. Then
we equivalently write the above equation as $E_{\ell}\leq bE_{\ell-1} + a$, which gives
$$
E_{\ell} \leq b^{\ell}E_1 + \frac{b^{\ell}+1}{b-1}a.
$$
With $[x_{i}] = [z_{i}^{(0)}]_k$, we see the following is only dependant on the structure of $\gi$ and $\gid$, 
\begin{equation*}
    \begin{aligned}
    \|[z_{i'}^{(\ell)}]_{k-\ell} - [z_{i'}^{(\ell)}]_{k-\ell}\|_2
    &\leq (\gams c_{\theta})^{\ell}\|\phi(L_{\gi})\|_2^{\ell}
    \|[x_{i}] - [x_{i'}]\|_2\\ &+\frac{(\gams c_{\theta})^{\ell}\|\phi(L_{\gi})\|_2^{\ell}+1}
    {\gams c_{\theta}\|\phi(L_{\gi})\|_2-1}
    \gams c_{\theta}c_z
    \|\phi(L_{\gi}) - \phi(L_{\gid})\|_2.
    \end{aligned}
\end{equation*}
\end{proof}
% \begin{equation}
% \label{Eq:zi_diff}
%     \begin{aligned}
%     &\|z_i - z_{i'}\|_2 
%     = \|(z_{i}^{(k)})_{0} - (z_{i'}^{(k)})_{0}\|_2\\
%     \leq& \gams c_{\theta}\|\phi(L_{\gi})\|_2
%     \|(z_{i}^{(k-1)})_{1} 
%     - (z_{i'}^{(k-1)})_{1}\|_2 
%     +\gams c_{\theta}c_z 
%     \|\phi(L_{\gi}) - \phi(L_{\gid})\|_2.
%     % &\leq Mc_x + c_x\gamma_{\theta}\|L_{\gi} - L_{\gid}\|_2
%     \end{aligned}
% \end{equation}

Since the the graph Laplacians are normalised, we have $\|\phi(L_{\gi})\|_2\leq c_L<\infty\;\forall i$. In addition, let $$\|x_{p,q}^i - x_{p,q}^{i'}\|_2\leq \sup_i\sup_{p,q} \|x_{p,q}^i - x_{p,q}^{i'}\|_2 
=\sup_i\|f(L_{\gi}) - f(L_{\gid})\|_2
:= \delta_x.$$ 
Hence, $\|[x_{i}] - [x_{i'}]\|_2 \leq \delta_x(\sum_{p=1}^k r^p)^{1/2}:=c_x$.
From Lemma \ref{lemma:z_expansion}, it is clear that we obtain the following at the final layer
\begin{equation}
    \label{Eq:Fnorm}
    \begin{aligned}
    I_2=\|z_i - z_{i'}\|_2 
    &\leq
    (\gams c_{\theta}c_L)^{k}c_x + 
    \frac{(\gams c_{\theta}c_L)^k+1}{\gams c_{\theta} c_L-1}\gams c_{\theta}c_z
    \|\phi(L_{\gi}) - \phi(L_{\gid})\|_2\\
    &\leq
    C(M c_x+\|L_{\gi} - L_{\gid}\|_2)\\
    &= C(M c_x+\lambda_{\max}(L_{\gi} - L_{\gid})^{1/2}).
    \end{aligned}
\end{equation}
since $\phi$ is a linear function for $L$. Indeed, this can be generalised to polynomial function $\phi$ of higher powers.

Now, consider the following term that is related with discriminator $\D$,
$$
I_1=
\sum_{p=1}^k 
    \sum_{q=1}^{|V_p(g_{i'})|}\left[
    \|h^i_{p,q} - h^{i'}_{p,q}\|_2
    + \|x^i_{p,q} - x^{i'}_{p,q}\|_2
    \right]/\sum_{p=1}^k r^p
$$
Firstly, 
% $$
% \sum_{p=1}^k 
%     \sum_{q=1}^{|V_p(g_{i'})|}
%     \|h_{p,q} - h'_{p,q}\|_2/\sum_{p=1}^k r^p
% $$
we denote $\Tilde{L}_{p,q}$ as the in-degree graph Laplacian derived with the subgraph $g_q$ of $g_i$ centred at $q\in V_p(g_i)$. Different from the encoder, we utilize every node's hidden embedding in the computation. Specifically, $g_q$ is obtained by retrieving links in $g_i$ that connects to the $q$th node in the $p$th layer. This is a principle submatrix of the in-degree graph Laplacian $\Tilde{L}_{g_i}$ of $g_i$. 
% The $\ell$-th update $[h_i^{(\ell)}]$ for all nodes of $g_i$ is given by a GCN $$[h_i^{(\ell)}] = \sigma(\phi(\Tilde{L}_{g_i})[h_i^{(\ell-1)}]\Tilde{\theta}^{(\ell)}).$$ 
% In this case, $p=\ell$.

Just as defined in \ssym\ref{Subsec:view}, we denote $[h_q^{(p)}]_{\ell}$ as the $p$th layer GCN embedding for nodes in hop 0 to hop $\ell\in [0,p]$ of $g_q$. Note that in this case, $[h_q^{(p)}]_0 = h_q^{(p)}$, which is one row of $[h_i^{(p)}]$, corresponding to the $q$-th node in the neighborhood.
So we may write the first term in $I_1$ as
$$
\sum_{p=1}^k 
    \sum_{q=1}^{|V_p(g_{i'})|}
    \|h_{q}^{(p)} - h_{q'}^{(p)}\|
$$
where $h_{q'}^{(p)}:=h_{p,q}^{i'}$ for short.
In this way, we regard each node $q\in V_p(g_i)$ as the centre node, which allows us to unfold the convolution similarly as expanding the $I_2$ term.
% \begin{equation*}
%     \begin{aligned}
%     \|h_q^{(k)} - h_{q}^{'(k)}\|
%     &= \|[\sigma(\phi(\Tilde{L}_{k,q})(h_q^{(k)})\Tilde{\theta}^{(k)}
%     - \sigma(\phi(\Tilde{L}'_{k,q})(h_{q}^{'(k)})\Tilde{\theta}^{(k)}]_0
%     \|\\
%     &\leq \gams c_{\Tilde{\theta}}
%     \|\phi(\Tilde{L}_{k,q})\|_2
%     \|(h_q^{(k-1)})_1 - (h_{q}^{'(k-1)})_1\|_2
%     + \gams c_{\Tilde{\theta}}c_h
%     \|\phi(\Tilde{L}_{k,q}) - 
%     \phi(\Tilde{L}'_{k,q})\|_2\\
%     &\leq \gams c_{\Tilde{\theta}}
%     \|\phi(\Tilde{L}_{k,q})\|_2
%     \|(h_q^{(k-1)})_1 - (h_{q}^{'(k-1)})_1\|_2
%     + \gams c_{\Tilde{\theta}}c_h
%     \|\phi(\Tilde{L}_{g_i}) - 
%     \phi(\Tilde{L}_{g_{i'}})\|_2
%     \end{aligned}
% \end{equation*}
Now, for any $q\in V_k(g_i)$, i.e. when $p=k$, we apply Lemma \ref{lemma:z_expansion} similarly as for $\|z_i - z_{i'}\|_2$. Then,
\begin{equation*}
    \begin{aligned}
    \|h_q^{(k)} - h_{q'}^{(k)}\|
    &\leq 
    (\gams c_{\Tilde{\theta}}c_{\Tilde{L}})^{k}c_x + 
    \frac{(\gams c_{\Tilde{\theta}}c_{\Tilde{L}})^k+1}{\gams c_{\Tilde{\theta}} c_{\Tilde{L}}-1}\gams c_{\Tilde{\theta}}c_h
    \|\phi(\Tilde{L}_{k,q}) - \phi(\Tilde{L}_{k,q'})\|_2\\
    &\leq \Tilde{C}_k(\Tilde{M}_k c_x + \|\phi(\Tilde{L}_{g_i}) - 
    \phi(\Tilde{L}_{g_{i'}})\|_2)\\
    \end{aligned}
\end{equation*}
where $\Tilde{L}_{p,q}$ is the principle submatrix of $\Tilde{L}_{g_{i}}$ and Lemma \ref{lemma:OptNorm} can be applied iin the last inequality. In addition, $\Tilde{C}_k$ and $\Tilde{M}_k$ are taken to be the maximum over any $q\in V_k(g_i)$.
In general, for $q\in V_{p}(g_i)$, $0<p<k$, we have 
$$
\|h_q^{(p)} - h_{q'}^{(p)}\|_2
\leq \Tilde{C}_{p}(\Tilde{M}_{p}c_x + \|\phi(\Tilde{L}_{g_i}) - 
    \phi(\Tilde{L}_{g_{i'}})\|_2)
$$
% since $\Tilde{L}_{p,q}$ is the principle submatrix of $\Tilde{L}_{g_{i}}$ and Lemma \ref{lemma:OptNorm} can be applied. 
Take a common upper bound for $\Tilde{C}_p, \Tilde{M}_p$ over $0< p\leq k$, we obtain
\begin{equation*}
    \begin{aligned}
    \sum_{p=1}^k\sum_{q=1}^{|V_p(g_{i'})|}
    \|h_q^{(p)} - h_{q'}^{(p)}\|/
    \sum_{p=1}^k r^p
    &\leq
    \Tilde{C}(\Tilde{M}c_x + \|\Tilde{L}_{g_i} - 
    \Tilde{L}_{g_{i'}}\|_2)\\
    &=\Tilde{C}(\Tilde{M}c_x + \lambda_{\max}(\Tilde{L}_{g_i} - 
    \Tilde{L}_{g_{i'}})^{1/2})
    \end{aligned}
\end{equation*}

In addition, for the other half of $I_1$, we have
$$
\sum_{p=1}^k 
    \sum_{q=1}^{|V_p(g_{i'})|}
    \|x_{p,q}^i - x_{p,q}^{i'}\|_2/\sum_{p=1}^k r^p
    \leq \sup_i\sup_{p,q} \|x_{p,q}^i - x_{p,q}^{i'}\|_2 = \delta_x
    = c_x/(\sum_{p=1}^k r^p)^{1/2}
$$

We can write $\B$ in terms of weights $C$ and $\Tilde{C}$, which is dependant on the activation function $\sigma$, $k$ and $\sup_i \lambda_{\max}(L_{g_i})$. Hence,
\begin{equation*}
    \begin{aligned}
    B 
    &\leq (CM+\Tilde{C}\Tilde{M}+1/(\sum_{p=1}^k r^p))c_x +
    C\lambda_{\max}(L_{\gi} - L_{\gid}) +
    \Tilde{C}\lambda_{\max}(\Tilde{L}_{g_i} - 
    \Tilde{L}_{g_{i'}})\\
    &= M'c_x+ C\lambda_{\max}(L_{\gi} - L_{\gid}) +
    \Tilde{C}\lambda_{\max}(\Tilde{L}_{g_i} - 
    \Tilde{L}_{g_{i'}})
    \end{aligned}
\end{equation*}

Note that the derived $I_1$ for $B$ is the same for $A$, since the node features, edge information and embedded features are bounded by separate terms in Eq.~\ref{Eq:B}. The only difference is given by $I_2$, where a different set of graph Laplacians $L_{g_j}, \,L_{g_{j'}}$ and node features $(x_j)$ are used. Therefore,
\begin{equation*}
    \begin{aligned}
    A
    &\leq M'c_x+ C\lambda_{\max}(L_{g_j} - L_{g_{j'}}) +
    \Tilde{C}\lambda_{\max}(\Tilde{L}_{g_i} - 
    \Tilde{L}_{g_{i'}})
    \end{aligned}
\end{equation*}
 Hence the result.
\end{proof}

Note that, our view of structural information is closely related to graph kernels \cite{bai2016subgraph} and graph perturbation \cite{verma2019stability}. Specifically, our Definition on k-hop ego-graph is motivated by the concept of k-layer expansion sub-graph in \cite{bai2016subgraph}. However, \cite{bai2016subgraph} used the Jensen-Shannon divergence between pairwise representations of sub-graphs to define a depth-based sub-graph kernel, while we depict $G$ as samples of its ego-graphs. In this sense, our view is related to the setup in \cite{verma2019stability}, which derived a uniform algorithmic stability bound of a 1-layer GNN under 1-hop structure perturbation of $G$. 

In the setting of domain adaptation, \cite{ben2007analysis} draws a connection between the difference in the distributions of source and target domains and the model transferability, and learns a transferable model by minimizing such distribution differences. This coincides with our approach of connecting the structure difference of two graphs in terms of k-hop subgraph distributions and the transferability of GNNs in the above theory.

\section{Model Details}
\label{supp:model}
Following the same notations used in the main paper, \Ours consists of a GNN encoder $\Psi$ and a GNN discriminator $\D$. In general, the GNN encoder $\Psi$ and discriminator $\D$ can be any existing GNN models. For each ego-graph and its node features $\{g_i, x_i\}$, the GNN encoder returns node embedding $z_i$ for the center node $v_i$. As mentioned in Eq. 2 in the main paper, the GNN discriminator $\D$ makes edge-level predictions as follows,
\begin{equation}
\label{eq:edge-msg}
    \D(e_{\tilde{v}v}|h_{p,q}^{\tilde{q}}, x^i_{p,q}, z_i) = \sigma \left( U^T \cdot \tau \left( W^T [ h_{p,q}^{\tilde{q}}|| x^i_{p,q}|| z_i] \right) \right),
\end{equation} 
where $e_{\tilde{v}v} \in E(g_i)$ and $h_{p,q}^{\tilde{q}} \in \mathbb{R}^d$ (simplified as $h_p$ in the main paper, same for $x^i_{p,q}=x_q$) is the representation for edge $e_{\tilde{v}v}$ between node $v_{p-1,\tilde{q}}$ in hop $p-1$ and $v_{p,q}$ in hop $p$. 
The prediction relies on the combination of center node embedding $z_i$, destination node feature $x^i_{p,q}$ and source node representation $h_{p,q}^{\tilde{q}}$. And now we describe how we calculate the source node representation in $\D$.

%Specifically, we denote the source node at $p-1$ hop as $\tilde{q} \in \Tilde{Q}_{p,q}, \Tilde{Q}_{p,q} = \{\Tilde{q}: v_{p-1,\Tilde{q}}\in V_{p-1}(g_i), e_{(p-1, \Tilde{q})(p, q)}\in E(g_i)\}$.
%\footnote{The notation of $e_{\tilde{v}v}$ here is a bit different from $e_{vv'}$ in the main paper, due to the convenience of using $\tilde{q}$ in the following. We will make them consistent in the later revisions.}  

\begin{algorithm*}[ht]
\label{alg}
%\SetAlgoLined
%\KwResult{Classifier $\{\theta\}$, Cluster GNN $\{\mathcal{E}, w, \{u_k\}\}$, assignment $\M$  }

The GNN encoder $\Psi$ and the GNN discriminator $\mathcal{D}$, k-hop ego graph and features $\{g_i, x_i\}$\;
 /* EGI-training starts */ \\
 \While{$\mathcal{L}_{\Ours}$  not converges}{
% /* Update open set semi-supervised learning */\\
  \text{Sample} M ego-graphs $\{ (g_1,x_1),...,(g_M, x_M) \}$ \text{from empirical distribution} $\mathbb{P}$ without replacement, and obtained their positive and negative node embeddings $z_i, z_i'$ through $\Psi$
  \[
  z_i = \Psi(g_i, x_i), z_i' = \Psi(g_i', x_i'), 
  \]
  /* Initialize positive and negative expectation in Eq. 1 in the main paper*/ \\
  $E_{pos} =0, E_{neg} = 0$ \\
  \For{p = 1 to $k$}{
    
    /* Compute JSD on edges at each hop*/ \\
    \For{$e_{(p-1, \Tilde{q})(p, q)} \in E(g_i)$}{
        generate source node embedding $h_{p,q}^{\tilde{q}}$ in Eq.~\ref{eq:update_h} \;
        $E_{\text{pos}} = E_{\text{pos}}$ + $\sigma \left(U^T \cdot \tau \left( W^T [ h_{p,q}^{\tilde{q}} || x^i_{p,q}|| z_i] \right) \right)$ \\
        $E_{\text{neg}} = E_{\text{neg}}$ + $\sigma \left( U^T \cdot \tau \left( W^T [ h_{p,q}^{\tilde{q}} || x^i_{p,q}|| z_i'] \right) \right)$ \\
    }

}
    /* Compute batch loss*/ \\
    $\mathcal{L}_{\text{EGI}} = E_{\text{neg}} - E_{\text{pos}}$ \\
    /* Update $\Psi$, $\D$ */ \\
    $\theta_{\Psi} \xleftarrow{+} - \nabla_{\Psi} \mathcal{L}_{\text{EGI}}$, $\theta_{\D} \xleftarrow{+} - \nabla_{\D}\mathcal{L}_{\text{EGI}}$

 }
 \caption{Pseudo code for training \Ours }
 \label{alg}

\end{algorithm*}

%\begin{figure}[ht]
%\centering
%\includegraphics[width=0.9\textwidth]{training_framework.pdf}
%\caption{The overall EGI training framework.}
%\label{fig:framework}
%\end{figure}

%In Figure~2 in the main paper, $\{g_i, x_i\}$ and $\{g_i', x_i'\}$ are the positive and negative training samples \textbf{w.r.t} ego-graph topology $g_i$.
To obtain the source node representation representations $h$, the GNN in discriminator $\D$ operates on a reversed ego-graph $\tilde{g_i}$ while encoder $\Psi$ performs forward propagation on $g_i$. The discriminator GNN starts from the center node $v_i$ and compute the hidden representation $m_{p-1,\tilde{q}}$ for node $v_{p-1, q}$ at each hop.  We denote the source node at $p-1$ hop as $\tilde{q} \in \Tilde{Q}_{p,q}, \Tilde{Q}_{p,q} = \{\Tilde{q}: v_{p-1,\Tilde{q}}\in V_{p-1}(g_i), e_{(p-1, \Tilde{q})(p, q)}\in E(g_i)\}$. Although $h_{p,q}$ is calculated as node embedding, in reversed ego graph $\tilde{g_i}$, node only has one incoming edge. Thus, we can also interpret $h_{p,q}^{\tilde{q}}$ as the edge embedding as it combines source node's hidden representation $m_{p-1,\Tilde{q}}$ and destination node features $x_{p,q}$ as follows,  
%For example, , the discriminator starts from the center node $v_i$. It first propagate 
\begin{equation}
\label{eq:update_h}
    h_{p,q}^{\tilde{q}} = \text{ReLU} \left( W_{p}^T \left(  m_{p-1,\Tilde{q}} + x^i_{p,q} \right)  \right), \;
    m_{p-1, \Tilde{q}} = \frac{1}{|\Tilde{Q}_{{p-1},\Tilde{q}}|}
\sum_{q'\in \Tilde{Q}_{{p-1}\Tilde{q}}} h_{p-1,\Tilde{q}}^{q'}
\end{equation} 

When $p=1$, every edge origins from the center node $v_i$ and $m_{0,q'}$ is the center node feature $x_{v_i}$. Note that we the elaborated aggregation rule is equivalent as layer-wise propagation rules (different in-degree matrix for each $h_{p,q}$) of \Ours earlier in \ssym\ref{Subsec:view}.

In every batch, we sample a set of ego-graphs and their node features $\{g_i, x_i\}$. During the forward pass of encoder $\Psi$, it aggregates from neighbor nodes to the center node $v_i$. Then, the discriminator calculates the edge embedding in Eq.~\ref{eq:update_h} from center node $v_i$ to its neighbors and make edge-level predictions-- \textit{fake} or \textit{true}. Besides training framework Figure 2 in the main paper, the algorithm \Ours is depicted in Algorithm \ref{alg}.

We implement our method and all of the baselines using the same encoders $\Psi$: 2-layer GIN~\cite{xu2019powerful} for synthetic and role identification experiments, 2-layer GraphSAGE~\cite{hamilton2017inductive} for the relation prediction experiments. 
We set hidden dimension as 32 for both synthetic and role identification experiments, For relation prediction fine-tuning task, we set hidden dimension as 256. 
We train \Ours in a mini-batch fashion since all the information for encoder and discriminators are within the k-hop ego-graph $g_i$ and its features $x_i$. Further, we conduct neighborhood sampling and set maximum neighbors as 10 to speed up the parrallel training. 
The space and time complexity of \Ours is $O(BN^K)$, where $B$ is the batch size, $N$ is the number of the neighbors and k is the number of hops of ego-graphs. Notice that both the encoder $\Psi$ and discriminator $\mathcal{D}$ propagate message on the k-hop ego-graphs, so the extra computation cost of $\mathcal{D}$ compared with a common GNN module is a constant multiplier over the original one. The scalability of \Ours on million scale YAGO network is reported in section \ref{sec:real-experiment}.

\subsection{Transfer Learning Settings}
\label{sec:transfer-learning-setup}
The goal of transfer learning is to train a model on a dataset or task, and use it on another. In our graph learning setting, we focus on training the model on one graph and using it on another. In particular, we focus our study on the setting of \textit{unsupervised-transfering}, where the model learned on the source graph is directly applied on the target graph without \textit{fine-tuning}. We study this setting because it allows us to directly measure the transferability of GNNs, which is not affected by the fine-tuning process on the target graph. In other words, the fine-tuning process introduces significant uncertainty to the analysis, because there is no guarantee on how much the fine-tuned GNN is different from the pre-trained one. Depending on specific tasks and labels distributions on the two graphs, the fine-tuned GNN might be quite similar to the pre-trained one, or it can be significantly different. It is then very hard to analyze how much the pre-trained GNN itself is able to help. Another reason is about efficiency. The fine-tuning of GNNs requires the same environment set-up and computation resource as training GNNs from scratch, although it may take less training time eventually if pre-training is effective. It is intriguing if this whole process can be eliminated when we guarantee the performance with unsupervised-transfering.

In our experiments, we also study the setting of transfer learning with fine-tuning, particularly on the real-world large-scale YAGO graphs.
Since we aim to study the general transferability of GNNs not bounded to specific tasks, we always pre-train GNNs with the unsupervised pre-training objective on source graphs. Then we enable two types of fine-tuning. The first one is \textit{post-fine-tuning} ($\mathcal{L} = \mathcal{L}_{s}$), where the pre-trained GNNs are fine-tuned with the supervised task specific objective $\mathcal{L}_{s}$ on the target graphs. The second on is \textit{joint-fine-tuning} ($\mathcal{L} = \mathcal{L}_{s} + \mathcal{L}_{u}$), where pre-training is the same, but fine-tuning is done \wrt~both the pre-training objective $\mathcal{L}_{u}$ and task specific objective $\mathcal{L}_{s}$ on target graphs in a semi-supervised learning fashion. 
The unsupervised pre-training objective $\mathcal{L}_{u}$ of \Ours is Algorithm 1, while those of the compared algorithms are as defined in their papers. The supervised fine-tuning objective $\mathcal{L}_{s}$ is the same as in the DistMult paper \cite{yang2014embedding} for all algorithms. %For \textit{joint-fine-tuning}, we adopt the same batch training as in pre-training by adding a supervised loss on all of the training samples in each batch, \ie, $\mathcal{L}_{\text{post-fine-tuning}} = \mathcal{L}_u +\mathcal{L}_s$. 

\section{Additional Experiment Details}
\subsection{Synthetic Experiments}
\label{supp:exp_syn}
\xhdr{Data.}
As mentioned in the main paper, we use two traditional graph generation models for synthetic data generation: (1) barabasi-albert graph~\cite{barabasi1999emergence} and (2) forest-fire graph~\cite{leskovec2005graphs}. We generate 40 graphs each with 100 nodes with each model. We control the parameters of two models to generate two graphs with different ego-graph distributions. Specifically, we set the number of attached edges as 2 for barabasi-albert model and set $p_\text{forward}=0.4$, $p_\text{backward}=0.3$ for forest-fire model. In Figure~\ref{fig:syn-a} and \ref{fig:syn-b}, we show example graphs from two families in our datasets. They have the same size but different appearance which leads to our study on the transferability gap $\Delta_{\D}$(F, F) and $\Delta_{\D}$(B, F) in Table 1 in the main paper. The accuracy of this task defined as the percentage of nearest neighbors for target node in the embedding space $z = \Psi(\cdot)$ that are structure-equivalent, \ie \#correct k-nn neighbors / \#ground truth equivalent nodes.

\begin{figure}[ht]
  \subfloat[Forest-fire graph example]{
	\begin{minipage}[c][1\width]{
	   0.3\textwidth}
	   \label{fig:syn-a}
	   \centering
	   \includegraphics[width=1\textwidth]{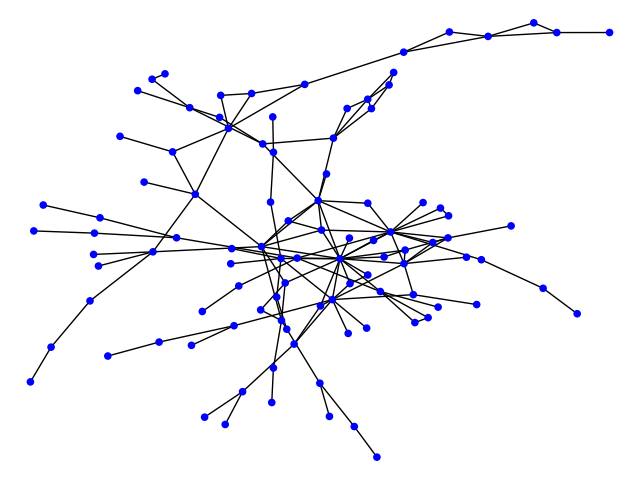}
	\end{minipage}}
 \hfill 	
  \subfloat[Barabasi-albert graph example]{
	\begin{minipage}[c][1\width]{
	   0.3\textwidth}
	   \label{fig:syn-b}
	   \centering
	   \includegraphics[width=1\textwidth]{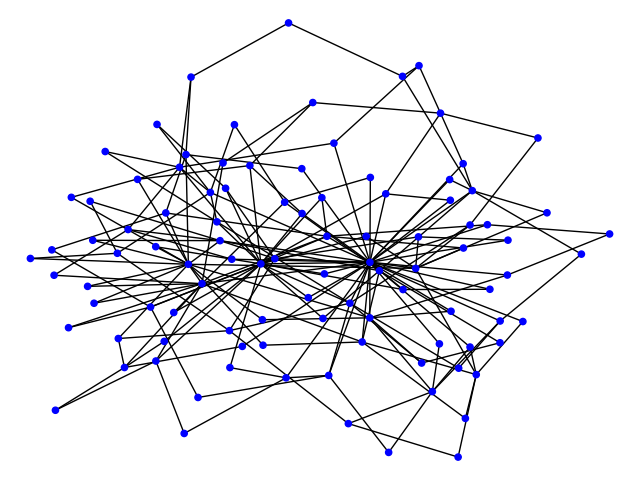}
	\end{minipage}}
 \hfill	
  \subfloat[structural label example]{
	\begin{minipage}[c][1\width]{
	   0.3\textwidth}
	   \label{fig:syn-c}
	   \centering
	   \includegraphics[width=1\textwidth]{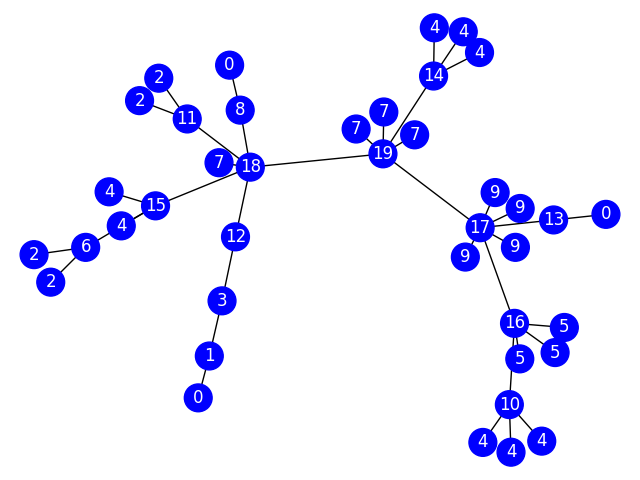}
	\end{minipage}}
\caption{Visualizations of the graphs and labels we use in the synthetic experiments.}
\label{fig:syn}
\end{figure}

\xhdr{Results.}
The structural equivalence label is obtained by a 2-hop WL-test \cite{weisfeiler1968reduction} on the ego-graphs. If two nodes have the same 2-hop ego-graphs, they will be assigned the same label.
In the example of Figure~\ref{fig:syn-c}, the nodes labeled with same number (\eg 2, 4) have the isomorphic 2-hop ego-graphs. Note that this task is exactly solvable when node features and GNN architectures are powerful enough like GIN~\cite{xu2019powerful}. In order to show the performance difference among different methods, we set the length of one-hot node degree encoding to 3 (all nodes with degrees higher than 3 have the same encoding). Here, we present the performance comparison with different length of degree encodings (d) in Table~\ref{tab:synthetic1}. When the capacity of initial node features is high (d=10), the transfer learning gap diminishes between different methods and different graphs because the structural equivalence problem can be exactly solved by neighborhood aggregations. However, when the information in initial node features is limited, the advantage of \Ours in learning and transfering the graph structural information is obvious. In Table~\ref{tab:synthetic2}, we also show the performance of different transferable and non-transferable features discussed after Definition 4.3 in the main paper, \ie node embedding~\cite{perozzi2014deepwalk} and random feature vectors. The observation is similar with Table 1 in the main paper: the transferable feature can reflect the performance gap between similar and dissimilar graphs while non-transferable features can not. 

In both Table~\ref{tab:synthetic1} and \ref{tab:supps-role-classification} here as well as Table 1 in the main paper, we report the structural difference among graphs in the two sets ($\bar{d}$) calculated \wrt~the term $\Delta_\D (G_a, G_b)$ on the RHS of Theorem 4.1 in the main paper. This indicates that the Forest fire graphs are structurally similar to the other Forest fire graphs, while less similar to the Barabasi graphs, as can be verified from Figure \ref{fig:syn-a} and \ref{fig:syn-b}. Our bound in Theorem 4.1 then tells us that the GNNs (in particular, \Ours) should be more transferable in the F-F case than B-F. This is verified in Table~\ref{tab:synthetic1} and~\ref{tab:synthetic2} when using the transferable node features of degree encoding with limited dimension (d=3) as well as DeepWalk embedding, as \Ours pre-trained on Forest fire graphs performs significantly better on Forest fire graphs than on Barabasi graphs (with +0.094 and +0.057 differences, respectively).

\begin{table*}[h]
\begin{center}
\caption{Synthetic experiments of identifying structural-equivalent nodes with different degree encoding dimensions.
}
\label{tab:synthetic1}
\scalebox{0.85}{
\begin{tabular}{lll|c|c|c|c|c|c|c|c}
\toprule
\multicolumn{3}{c|}{\multirow{2}{*}{Method}}                             & \multicolumn{3}{c|}{\textbf{\#dim degree encoding d = 3}}            & \multicolumn{3}{c|}{\textbf{\# dim degree encoding d = 10}} & \multicolumn{2}{c}{\textbf{structural difference   }}                                                                                                                   \\
\multicolumn{3}{c|}{}                                                    & F-F & B-F & $\delta$(acc.) & F-F & B-F & $\delta$(acc.)  & $\Delta_{D}$(F,F) & $\Delta_{D}$(B,F)\\ \midrule
%\multicolumn{2}{c}{\multirow{3}{*}{}}             & GraphSAGE        &   82.2 $\pm$ 1.1\%      &    N.A.            &  N.A.        &  N.A.     &   83.1 $\pm$ 0.8\%      &    N.A.            &  N.A.        &  N.A.                                    \\
%\midrule
\multicolumn{2}{l}{}                              & GCN (untrained)                    &  0.478       &   0.478       &  /       &  0.940 & 0.940 & /  & \multirow{5}{*}{0.752}  &   \multirow{5}{*}{0.883}                             \\
\multicolumn{2}{l}{}                              & GIN (untrained)                    &  0.572       &   0.572       &  /       &  0.940 & 0.940 & / &  &                             \\
\multicolumn{2}{l}{}                              & VGAE (GIN)                   &   0.498     &   0.432  &  +0.066       &    0.939      & 0.937  & 0.002 & &                          \\
\multicolumn{2}{l}{}                              & DGI (GIN)                    &  0.578       &    0.591     &  -0.013      & 0.939    &  0.941   & -0.002 &    &                    \\
\multicolumn{2}{l}{}                              & \Ours (GIN)                    & \textbf{0.710}       &    0.616        &   +0.094    &  0.942     & 0.942 &    0   & &                  \\
\bottomrule

\end{tabular}
}
\end{center}
\end{table*}

\begin{table*}[h]
\begin{center}
\caption{Synthetic experiments of identifying structural-equivalent nodes with different transferable and non-transferable features.
}
\label{tab:synthetic2}
\scalebox{0.85}{
\begin{tabular}{lll|c|c|c|c|c|c|c|c}
\toprule
\multicolumn{3}{c|}{\multirow{2}{*}{Method}}                             & \multicolumn{3}{c|}{\textbf{DeepWalk embedding}}            & \multicolumn{3}{c|}{\textbf{random vectors}} & \multicolumn{2}{c}{\textbf{structural difference   }}                                                                                                                   \\
\multicolumn{3}{c|}{}                                                    & F-F & B-F & $\delta$(acc.) & F-F & B-F & $\delta$(acc.)  & $\Delta_{D}$(F,F) & $\Delta_{D}$(B,F)\\ \midrule
%\multicolumn{2}{c}{\multirow{3}{*}{}}             & GraphSAGE        &   82.2 $\pm$ 1.1\%      &    N.A.            &  N.A.        &  N.A.     &   83.1 $\pm$ 0.8\%      &    N.A.            &  N.A.        &  N.A.                                    \\
%\midrule
\multicolumn{2}{l}{}                              & GCN (untrained)                    &  0.658       &   0.658        &  /       &  0.246 & 0.246 & /  & \multirow{5}{*}{0.752}  &   \multirow{5}{*}{0.883}                             \\
\multicolumn{2}{l}{}                              & GIN (untrained)                    &  0.663       &   0.663       &  /       &  0.520 &  0.520 & / &  &                             \\
\multicolumn{2}{l}{}                              & GVAE (GIN)                   &   0.713     &   0.659  &  +0.054       &    0.266      & 0.264  & 0.002 & &                          \\
\multicolumn{2}{l}{}                              & DGI (GIN)                    &  0.640       &    0.613     &  +0.027      & 0.512         & 0.576 &   -0.064 & &                       \\
\multicolumn{2}{l}{}                              & \Ours (GIN)                    & \textbf{0.772}       &    0.715        &   +0.057    &    0.507   & 0.485 &    +0.022    & &                  \\
\bottomrule

\end{tabular}
}
\end{center}
\end{table*}

\subsection{Real-world Role Identification Experiments}
\label{supp:exp_airport}
\xhdr{Data.}
\begin{table}[t]
    \caption{Overall Dataset Statistics}
    \label{tab:dataset-stats}
    \centering
    \begin{tabular}{c|c c c c }
    \toprule
         Dataset & \# Nodes & \# Edges & \# Classes \\
         \midrule 
         Europe & 399 & 5,995  & 4   \\
         USA & 1,190  &   13,599 & 4 \\
         Brazil & 131  &   1,074 & 4 \\
         Gene & 9,228  &   57,029 & 2 \\
         \bottomrule
    \end{tabular}

    %\vspace{-0.5cm}
\end{table}
% \carl{here we talk about the airport and gene datasets, this label generated by WL should be moved to synthetic experiment?}
We report the number of nodes, edges and classes for both airport and gene dataset. The numbers for the Gene dataset are the aggregations of the total 52 gene networks in the dataset. For the three airport networks, Figure~\ref{fig:power-law} shows the power-law degree distribution on log-log scale. 
The class labels are between 0 to 3 reflecting the level of the airport activities \cite{ribeiro2017struc2vec}.
For the Gene dataset, we matched the gene names in the TCGA dataset~\cite{yang2019conditional} to the list of transcription factors on wikipedia\footnote{\url{https://en.wikipedia.org/wiki/Transcription_factor}}. 75\% of the genes are marked as 1 (transcription factors) and some gene graphs have extremely imbalanced class distributions. So we conduct experiments on the relatively balanced gene graphs of brain cancers (Figure 2 in the main paper). Both datasets do not have organic node attributes. The role-based node labels are highly relevant to their local graph structures, but are not trivially computable such as from node degrees.

\begin{figure}[ht]
  \subfloat[Europe airport log-log plot]{
	\begin{minipage}[c][1\width]{
	   0.32\textwidth}
	   \label{fig:euro}
	   \centering
	   \includegraphics[width=1\textwidth]{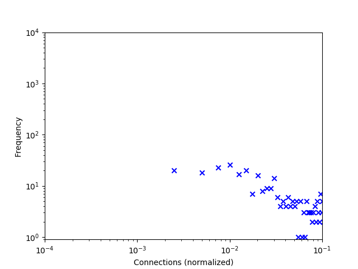}
	\end{minipage}}
 \hfill 	
  \subfloat[USA airport log-log plot]{
	\begin{minipage}[c][1\width]{
	   0.32\textwidth}
	   \label{fig:usa}
	   \centering
	   \includegraphics[width=1\textwidth]{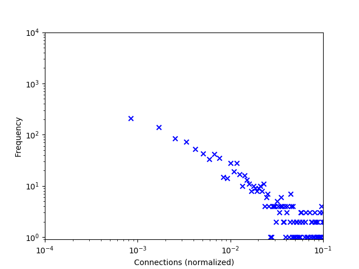}
	\end{minipage}}
 \hfill	
  \subfloat[Brazil airport log-log plot]{
	\begin{minipage}[c][1\width]{
	   0.32\textwidth}
	   \label{fig:brazil}
	   \centering
	   \includegraphics[width=1\textwidth]{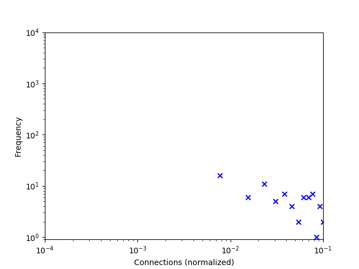}
	\end{minipage}}
\caption{Visualizations of power-law degree distribution on three airport dataset.}
\label{fig:power-law}
\end{figure}

\xhdr{Results.}
As we can observe from Figure~\ref{fig:power-law}, the three airport graphs have quite different sizes and structures (\eg, regarding edge density and connectivity pattern). Thus, the absolute classification accuracy in both Table 2 in the main paper and Table~\ref{tab:supps-main-table2} here varies across different graphs. However, as we mention in the main paper, the structural difference we compute based on Eq.~5 in Theorem 3.1 is close among the Europe-USA and Europe-Brazil graph pairs (0.869 and 0.851), which leads to close transferability of \Ours from Europe to USA and Brazil. This indicates the effectiveness of our view over essential structural information. In Table~\ref{tab:supps-main-table2}, we also provide the comparison between transferable and non-transferable feature on airport dataset. As expected, \Ours only yields good transferability with transferable features.

\begin{table*}[h]
\begin{center}

\caption{Results of role identification with direct-transfering on the Airport dataset (Table 2, main paper). The performance reported (\%) are the average over 100 runs. We set all node features same as non-transferable features.}
\label{tab:supps-main-table2}
\scalebox{0.9}{
\begin{tabular}{lll|c|c|c|c|c|c}
\toprule
\multicolumn{3}{c|}{\multirow{2}{*}{Method}}                             & \multicolumn{2}{c|}{Europe (source)} & \multicolumn{2}{c|}{USA (target)} & \multicolumn{2}{c}{Brazil (target)}                                                                                                            \\
\multicolumn{3}{c|}{}                                                    & node degree & same feat. & node degree & same feat. & node degree & same feat. \\ \midrule                    

\multicolumn{2}{l}{}                              & features &  52.81         &   20.59       &   55.67      &  20.22 & 67.11 & 19.63 \\

%\multicolumn{2}{l}{}                              & GCN (untrained)                    &  52.96       &  20.11       &  55.30       &  22.07 & 68.30 & 17.63  \\

\multicolumn{2}{l}{}                              & GIN (untrained)                    &  55.75   &   53.88       &    61.56      &  58.32  & 70.04 & 70.37\\
\multicolumn{2}{l}{}                              & GVAE                  &  53.90    &   21.12      &  55.51       &  22.39 & 66.33 & 17.70                            \\
\multicolumn{2}{l}{}                              & DGI                   &  57.75      &    22.13     &  54.90      &  21.76 & 67.93 & 18.78                          \\
\multicolumn{2}{l}{}                              & MaskGNN                   &  56.37      &  55.53        &      60.82 & 54.64 & 66.71 &74.54                     \\
\multicolumn{2}{l}{}                              & ContextPredGNN                    &  52.69       &   49.95       &  50.38      & 54.75  & 62.11 & 70.66                       \\
\multicolumn{2}{l}{}                              & Structural Pre-train                   & 56.00       &   53.83      &    62.17    & 57.49  & 68.78 &  72.41                     \\
\multicolumn{2}{l}{}                              & MVC                    &    53.16    &   51.69     &  59.66  & 50.42   &  66.07   & 61.55                    \\
\multicolumn{2}{l}{}                              & GMI                   & 58.12       &   46.25      &   59.28     & 47.64  & 73.07   &  62.96                     \\
\multicolumn{2}{l}{}                              & \Ours (GIN)                   & \textbf{59.15}$^{**}$       &  54.98 &\textbf{64.55}$^{**}$      &   57.40    &  \textbf{73.15} & 70.00                       \\
\bottomrule

\end{tabular}
}
\end{center}
\end{table*}

Besides that, the results present in Table~\ref{tab:supps-role-classification} are the accuracy of GNNs directly trained and evaluated on each network without transfering. Therefore, only the Europe column has the same results as in Table 2 in the main paper, while the USA and Brazil columns can be regarded as providing an upper-bound performance of GNN transfered from other graphs. As we can see, \Ours gives the closest results from Table 2 (main paper) to Table~\ref{tab:supps-role-classification} here, demonstrating the its plausible transferability. The scores are so close, showing a possibility to skip fine-tuning when the source and target graphs are similar enough. Also note that, although the variances are pretty large (which is also observed in other works like \cite{ribeiro2017struc2vec} since the networks are small), our t-tests have shown the improvements of \Ours to be significant.

\begin{table*}[h]
\begin{center}

\caption{Role identification that identifies structurally similar nodes on real-world networks. The performance reported are the average and standard deviation for 10 runs. Our classification accuracy on three datasets all passed the t-test (p$<$0.01) with the second best result in the table.}
\label{tab:supps-role-classification}
\scalebox{0.9}{
\begin{tabular}{lll|c|c|c}
\toprule
\multicolumn{3}{c|}{\multirow{2}{*}{Method}}                             & \multicolumn{3}{c}{\textbf{Airport~\cite{ribeiro2017struc2vec}}}                                          \\
\multicolumn{3}{c}{}                                                    & Europe & USA & Brazil \\ \midrule

\multicolumn{2}{l}{}                              & node degree                    &    52.81\% $\pm$ 5.81\%      &   55.67\% $\pm$ 3.63\%      &   67.11\% $\pm$ 7.58\%     \\
\multicolumn{2}{l}{}                              & GCN (random-init)                    &    52.96\% $\pm$ 4.51\%    &   56.18\% $\pm$ 3.82\%       &    55.93\% $\pm$ 1.38\%  \\
\multicolumn{2}{l}{}                              & GIN (random-init)                    &    55.75\% $\pm$ 5.84\%     &   62.77\% $\pm$ 2.35\%       &  69.26\% $\pm$ 9.08\%   \\
\multicolumn{2}{l}{}                              & GVAE (GIN)                   &  53.90\% $\pm$ 4.65\%     &   58.99\% $\pm$ 2.44\%      &  55.56\% $\pm$ 6.83\%     \\
\multicolumn{2}{l}{}                              & DGI (GIN)                    &    57.75\% $\pm$ 4.47\%     &    62.44\% $\pm$ 4.46\%      &  68.15\% $\pm$ 6.24\%       \\

\multicolumn{2}{l}{}                              & Mask-GIN                 &    56.37\% $\pm$ 5.07\%     &  63.78\% $\pm$ 2.79\%     & 61.85\% $\pm$ 10.74\%  \\
\multicolumn{2}{l}{}                              & ContextPred-GIN                   &    52.69\% $\pm$ 6.12\%     &   56.22\% $\pm$ 4.05\%       &  58.52\% $\pm$ 10.18\%    \\
\multicolumn{2}{l}{}                              & Structural Pre-train                   &   56.00\% $\pm$ 4.58\%    &   62.29\% $\pm$ 3.51\%      &    71.48\% $\pm$ 9.38 \% \\

\multicolumn{2}{l}{}                              & MVC                    &    53.16\% $\pm$ 4.07\%        &  62.81 \% $\pm$ 3.12\%   &  67.78 \% $\pm$ 4.79\%                \\

\multicolumn{2}{l}{}                              & GMI                   & 58.12 \% $\pm$ 5.28\%       &  63.36 \% $\pm$ 2.92\%     & 73.70\% $\pm$ 4.21\%                      \\

\multicolumn{2}{l}{}                              & \Ours (GIN)                    &   \textbf{59.15\% $\pm$ 4.44}\%     &  \textbf{65.88\% $\pm$ 3.65}\%       &   \textbf{74.07\% $\pm$ 5.49}\%    \\
\bottomrule

\end{tabular}
}
\end{center}
\end{table*}

\begin{table}[]
    \caption{dataset statistics and running time of \Ours}
    \label{tab:dataset-stats-2}
    \centering
\scalebox{0.9}{
    \begin{tabular}{c|c c c c c}
    \toprule
         Dataset & \# Nodes & \# Edges & \# Relations & \# Train/Test & Training time per epoch \\
         \midrule 
         YAGO-Source & 579,721  &   2,191,464 & / &  / & 338 seconds\\
         YAGO-Target & 115,186 & 409,952  & 24  & 480/409,472 & 134 seconds\\
         
         \bottomrule
    \end{tabular}
}
    %\vspace{-0.5cm}
\end{table}

\subsection{Real-world large-scale Relation Prediction Experiments}
\label{supp:exp_yago}
\label{sec:real-experiment}
\xhdr{Data.} As shown in Table~\ref{tab:dataset-stats-2}, the source graph we use to pre-train GNNs is the full graph cleaned from the YAGO dump \cite{suchanek2007yago}, where we assume the relations among entities are unknown. The target graph we use is a subgraph uniformed sampled from the same YAGO dump (we sample the nodes and then include all edges among the sampled nodes). The similar ratio between number of nodes and edges can be observed in Table~\ref{tab:dataset-stats-2}.
On the target graph, we also have the access to 24 different relations~\cite{shi2018easing} such as \textit{isAdvisedBy}, \textit{isMarriedTo} and so on. Such relation labels are still relevant to the graph structures, but the relevance is lower compared with the structural role labels. We use the 256-dim degree encoding as node features for pre-training on the source graph, then we use the 128-dim positional embedding generated by LINE~\cite{tang2015line} for fine-tuning on the target graph, to explicitly make the features differ across source and target graphs.

\xhdr{Results.} In Section~\ref{sec:transfer-learning-setup}, we introduced two different types of fine-tuning, \ie , \textit{post-fine-tuning} and \textit{joint-fine-tuning}. For both types of fine-tuning, we add one feature encoder $\mathcal{E}$ before feeding it into the GNNs for two purposes. First, the target graph fine-tuning feature usually has different dimensions with the pre-training features, such as the node degree encoding we use. Second, the semantics and distributions of fine-tuning features can be different from pre-training features. The feature encoder aims to bridge the gap between feature difference in practice. The supervised loss used in this experiment is the same as in DistMult~\cite{yang2014embedding}. In particular, the bilinear score function is calculated as $s(h,r,t) = z_h^T M_r z_t$, where $M_r$ is a diagonal matrix for each relation $r$, $z_h$ and $z_t$ the the embedding of GNN encoder $\Psi$ for head and tail entities. 
The experiments were run on GTX1080 with 12G memories. We report the average training time per epoch of our algorithm in pre-training and fine-tuning stage in Table~\ref{tab:dataset-stats-2} as well. 
The pre-training and fine-tuning takes about 40 epochs and 10 epochs to converge, respectively. 
In Table \ref{tab:dataset-stats-2}, we also present the per-epoch training time of \Ours. \Ours takes about 338 seconds per epoch for optimizing the ego-graph information maximization objective on YAGO-source.
As we can see, fine-tuning also takes significant time compared to pre-training, which strengthens our arguments about avoiding or reducing fine-tuning through structural analysis. We implement all baselines within the same pipeline, and the running times are all in same scale.

\subsection{Parameter study}
\label{supp:exp_para}
In this section, we provide additional parameter analysis towards proposed EGI model - choices of $k$, and efficiency study on EGI gap $\Delta_\D$ - sampling frequencies.

\xhdr{Performance of different size of ego-graphs.}
In our Theorem~\ref{theo:main} and EGI algorithm (Eq.~\ref{eq:EGI}), number of hops $k$ determines the size of ego-graphs. In principle, $k$ may affect the transferability of EGI in two ways: (1) larger $k$ may make the EGI model (both center node encoder $\Psi$ and neighbor node encoder $\Phi$) more expressive (better precision) and the EGI gap $\Delta_\D$ more accurate (better predictiveness); (2) However, the GNN encoders may suffer from the over-smoothing problem and the computations may suffer from more noises. 
Therefore, it is hard to determine the influence of $k$ without empirical analysis. As we can observe in , when $k=1$ or $k=3$, the classification accuracy of the source graph is worse than $k=2$, likely because the GNN encoder is either less powerful or over-smoothed. As a result,  $k=2$ obtains the best transferability to both the USA and Brazil networks. When $k=3$, $\Delta_\D$ likely accounts for too subtle/noisy ego-graph differences and may become less effective in predicting the transferability. Therefore, we choose $k=2$ to conduct experiments in main paper.
\begin{table}
\begin{center}
\caption{Comparison of EGI with different $k$. Accuracy and EGI gap $\Delta_\D$ are reported.}
\label{tab:supps-differnt-k}
\scalebox{0.9}{
\begin{tabular}{lll|c|c|c|c|c}
\toprule
\multicolumn{3}{c|}{\multirow{2}{*}{}}                             & Europe (source) & \multicolumn{2}{c|}{USA (target)} & \multicolumn{2}{c}{Brazil (target)}                                                                                                            \\
\multicolumn{3}{c|}{}                                                    & acc. & acc.& $\Delta_\D$ & acc.& $\Delta_\D$ \\ \midrule                    

\multicolumn{2}{l}{}                              & EGI (k=1) &  58.25         &   60.08       &   0.385      &  60.74 & 0.335  \\

%\multicolumn{2}{l}{}                              & GCN (untrained)                    &  52.96       &  20.11       &  55.30       &  22.07 & 68.30 & 17.63  \\

\multicolumn{2}{l}{}                              & EGI (k=2)                    &  59.15   &   64.55       &    0.869      &  73.15  & 0.851 \\
\multicolumn{2}{l}{}                              & EGI (k=3)                  &  57.63    &   64.12      &  0.912       &  72.22 & 0.909   \\
\bottomrule

\end{tabular}
}
\end{center}
\end{table}

\xhdr{Precision of $\Delta_\mathcal{D}$ under different sampling frequencies.} In Table~\ref{tab:supps-sample-frequency}, we present the estimated $\Delta_\D$ versus sampling frequency for 10 runs on airport dataset. A theoretical study on its convergence could be an interesting future direction. As we can observe, large sample frequency leads to more accurate and robust estimation of $\Delta_\D$. Between Europe and USA, although 100 pairs of ego-graphs are only equivalent as 2.1\% of the total pair-wise enumerations, the estimated  $\Delta_\D$ is pretty close.

\begin{table}[h!]
\begin{center}
\caption{EGI gap $\Delta_\D$ on airport dataset with different sampling frequencies.}
\label{tab:supps-sample-frequency}
\scalebox{0.9}{
\begin{tabular}{l|c|c}
\toprule
Sampling frequency & $\Delta_\D$(Europe, USA)& $\Delta_\D$(Europe, Brazil)\\
\midrule
100 pairs & 0.872$\pm$0.039 & 0.854$\pm$0.042\\
1000 pairs & 0.859$\pm$0.012 & 0.848$\pm$0.007 \\
All pairs & 0.869 $\pm$0.000 & 0.851 $\pm$0.000 \\
\bottomrule

\end{tabular}
}
\end{center}
\end{table}

\end{document}